\documentclass{article}
\usepackage{iclr2025_conference,times}
\usepackage[utf8]{inputenc} % allow utf-8 input
\usepackage[T1]{fontenc}    % use 8-bit T1 fonts
\usepackage{url}            % simple URL typesetting
\usepackage{booktabs}       % professional-quality tables
\usepackage{amsfonts}       % blackboard math symbols
\usepackage{nicefrac}       % compact symbols for 1/2, etc.
\usepackage{microtype}      % microtypography
\usepackage{xcolor}         % colors

\usepackage{amssymb}
\usepackage{natbib}
\usepackage{color}
\definecolor{dark-blue}{RGB}{0,0,191}
\usepackage[hidelinks,colorlinks=true, linkcolor={dark-blue}, citecolor={dark-blue}]{hyperref}
\usepackage{amsfonts}       
\usepackage{url}
\usepackage{amsmath}
\usepackage{amssymb}
\usepackage{amsthm}
\usepackage{wrapfig}
\usepackage{mathtools}
\usepackage{tikz}
\usepackage{subfig}
\usepackage{algorithmic}
\usepackage{subfig}
\usepackage{footmisc}
\usepackage{bibentry}
\nobibliography*
\usepackage{mathtools}

\usepackage{physics}
\usepackage{mathrsfs}
\mathtoolsset{showonlyrefs}
\usepackage{etoolbox}
\usepackage{makecell}
\usepackage[ruled]{algorithm2e}
\DontPrintSemicolon
\usepackage{enumerate}
\usepackage{booktabs}
\usepackage{lastpage}
\usepackage{float}
\usepackage{tabularx}
\usepackage{multirow}

\usepackage{thmtools, thm-restate}
\newcommand{\explain}[1]{\tag*{(#1)}}

\newcommand{\fS}{\mathcal{S}}
\newcommand{\fA}{\mathcal{A}}

\newcommand{\R}{\mathbb{R}}

\newcommand{\E}{\mathbb{E}}
\newcommand{\V}{\mathbb{V}}

\newcommand{\pdisg}{G^{\text{PDIS}}}

\newcommand{\tb}[1]{{\textbf{#1}}}

\def\argmin{\qopname\relax n{argmin}}

\newtheorem{theorem}{Theorem}%[section]
\newtheorem{lemma}{Lemma}%[Lemma]

\allowdisplaybreaks
 \iclrfinalcopy

\title{Efficient Off-Policy Evaluation with Safety \\ Constraint
for Reinforcement Learning}
\hypersetup{linkcolor=black}

\author{Claire Chen\thanks{Equal Contribution.}\\
School of Arts and Science\\
University of Virginia\\
\texttt{clairechen@email.virginia.edu} \\
\And
Shuze Daniel Liu\footnotemark[1] \\
Department of Computer Science \\
University of Virginia\\
\texttt{shuzeliu@virginia.edu} \\
\And
Shangtong Zhang\\
Department of Computer Science \\
University of Virginia\\
\texttt{shangtong@virginia.edu} \\
}

\begin{document}

\maketitle

\begin{abstract}
In reinforcement learning, classic on-policy evaluation methods often suffer from high variance and require massive online data to attain the desired accuracy. Previous studies attempt to reduce evaluation variance by searching for or designing proper behavior policies to collect data. However, these approaches ignore the safety of such behavior policies---the designed behavior policies have no safety guarantee and may lead to severe damage during online executions. In this paper, to address the challenge of reducing variance while ensuring safety simultaneously, we propose an optimal variance-minimizing behavior policy under safety constraints. Theoretically, while ensuring safety constraints, our evaluation method is unbiased and has lower variance than on-policy evaluation. Empirically, our method is the only existing method to achieve both substantial variance reduction and safety constraint satisfaction. Furthermore, we show our method is even superior to previous methods in both variance reduction and execution safety.

\end{abstract}

\section{Introduction}
Recently, reinforcement learning (RL, \citet{sutton2018reinforcement}) has shown exceptional success in various sequential decision-making problems. For instance, the applications of RL algorithms have reduced energy consumption for Google's data center's cooling by $40\%$ \citep{chervonyi2022semianalytical}, solved Olympiad-level geometry problems \citep{trinh2024solving}, and designed general-purpose data center CPUs \citep{mirhoseini2021graph}. 
In RL applications, \textit{policy evaluation} 
enables RL practitioners to estimate the performance of a policy before committing to its full deployment. In policy evaluation, conventional wisdom uses the on-policy method, in which a policy (i.e., the target policy) is evaluated by directly executing itself. However, this straightforward approach is crude since using the target policy itself as the data-collecting policy (i.e., the behavior policy) is proved to be suboptimal \citep{liu2024efficient, liu2024doubly, liu2024multi}, resulting in evaluations with potentially high variance. Thus, the on-policy evaluation method may require massive online data to achieve the desired accuracy.

Unfortunately, collecting massive online data in real-world interaction can be both expensive and slow \citep{li2019perspective, Zhang_2023}. In Google's data center's cooling system, each interaction step in the actual deployment takes $5$ minutes \citep{chervonyi2022semianalytical}. Thus, the evaluation of a policy requiring millions of steps is prohibitively expensive.
To reduce the reliance on costly online data collection, offline RL has been introduced as a possible solution. However, mismatches between the offline data distribution and the distribution induced by the target policy frequently arise, resulting in bias that is both uncontrolled and difficult to eliminate \citep{jiang2015doubly, farahmand2011model, marivate2015improved}. As a result, both online and offline RL practitioners still depend heavily on online policy evaluation techniques \citep{kalashnikov2018scalable, vinyals2019grandmaster}.

To improve the online sample efficiency for policy evaluation, existing methods propose to reduce the evaluation variance by searching for or designing proper behavior policies \citep{hanna2017data, zhong2022robust, liu2024efficient}. However, all their approaches ignore a critical issue:\textit{ safety}. In many real-world applications, neglecting safety in policy execution can result in serious consequences. For example, in Google's data center cooling system, a behavior policy that is tailored for variance reduction without considering safety constraints may unpredictably overheat the system, causing equipment damages or service disruptions. Therefore, besides reducing evaluation variance, it is crucial to guarantee execution safety. 

In this paper, we address the challenge of reducing variance while ensuring safety simultaneously. We make the following contributions:
\begin{enumerate}
\item We propose an optimal variance-minimizing behavior policy under safety constraints.
\item Theoretically, we show that our method gives an unbiased estimation. In addition to strictly satisfying the safety constraints, our method is proven to attain lower variance than the classic on-policy evaluation method.
\item Empirically, we show that our method is the only existing method 
to achieve both substantial variance reduction and constraint satisfaction.
Moreover, it is even superior to previous methods in both variance reduction and execution safety.

\end{enumerate}

\section{Related Work}\label{sec: related work} 
\textbf{Safe RL.}
Safety in reinforcement learning, often framed as safe RL \citep{garcia2015comprehensive}, has been an active research topic recently.
Many recent works focus on safety in policy exploration and optimization \citep{brunke2022safe}. For safe exploration,
\citet{moldovan2012safe} present a method for ensuring safe exploration by keeping the agent within a predefined set of safe states during its learning process. However, their method is a model-based approach, requiring an explicit approximation of the transition function,  which introduces challenges common to model learning, such as compounding errors and the need for accurate model dynamics \citep{sutton1990integrated,sutton2012dyna,deisenroth2011pilco,chua2018deep}. \textit{In contrast, our approach does not rely on approximating the transition function (i.e., model-free), since parameters can be estimated by 
off-the-shelf offline policy evaluation methods (e.g. Fitted Q-Evaluation, \citet{le2019batch}).}
As for safe optimization, \citet{berkenkamp2017safe} propose to ensure safety by keeping the agent within safe regions, which are characterized by a Lyapunov function. However, they assume the environment to be deterministic, i.e., $p(s'|s,a)=1$ for the successive state $s'$, which is a significant limitation as most MDPs
are stochastic. Their method is also model-based, requiring knowledge of the transition functions. \textit{In contrast, our approach copes well with stochastic environments and is model-free.} 

Safe reinforcement learning is often modeled as a Constrained Markov
Decision Process (CMDP) \citep{gu2022review, wachi2024survey, liu2021policy}, in which we need to maximize the agent reward
while making agents satisfy safety constraints. \citet{achiam2017constrained} enforce a constant threshold to constrain the expected total cost. 
However, even though they adopt the trust-region method to control policy updates, the expected total cost of the new policy can still exceed the safety threshold at each update step, leading to uncontrolled violations of the safety constraints over time. \textit{In contrast, our method inherently integrates safety constraints into the policy design, ensuring strict satisfaction of constraints throughout execution.} 
\citet{wachi2020safe} 
propose a method for safe reinforcement learning in constrained Markov decision processes (CMDPs) by using a Gaussian Process to model the safety constraints and guide exploration. Nevertheless, their approach needs to compute the covariance matrix between explored states throughout the execution, which is computationally expensive, especially in environments with large state spaces. In addition, they assume that the state transitions are deterministic, making their method highly restricted. \textit{In contrast, our algorithm does not rely on knowledge about the complicated covariances and copes with stochastic environments.}

\textbf{Variance Reduction. }
Variance reduction for policy evaluation in reinforcement learning (RL) is also widely explored. Since using the target policy as the data-collecting policy (i.e., the behavior policy) for evaluating itself is not optimal \citep{owen2013monte}, some recent studies focus on searching for or designing a data-reducing behavior policy without considering safety. 
\citet{hanna2017data} formulate the task of searching for a variance-reduction behavior policy as an optimization problem. They parameterize the behavior policy and use stochastic gradient descent to update the policy. 
However, in reinforcement learning, the stochastic method has been known to easily get stuck in highly suboptimal points in just moderately complex environments \citep{williams1992simple}, where various local optimal and saddle points exist.
Besides, they do not consider the safety of their obtained behavior policy, which might cause damage during execution. \textit{In contrast, our method directly learns the globally optimal behavior policy with safety guarantees.}
Moreover, to learn the behavior policy effectively, their method requires highly sensitive hyperparameter tuning, where the learning rate can vary by up to $10^5$ times across different environments, as reported by \citet{hanna2017data}. This  sensitivity requires online tuning, which consumes online data.
\textit{In contrast, we introduce an efficient algorithm to learn our behavior policy with purely offline data. }Furthermore, their methods need the online data to be complete trajectories. \textit{In contrast, our algorithm copes well with incomplete offline tuples, making it widely applicable.} 

\citet{zhong2022robust} also design a variance-reducing behavior policy for policy evaluation. They adjust the behavior policy to focus on under-sampled data segments. However, their method requires the offline data to be in complete trajectories, and the data must be generated by known behavior policies that are highly similar to the target policy. These assumptions are strong.
\textit{ In contrast, our method effectively handles incomplete offline segments from potentially unknown and diverse behavior policies.} 
\citet{zhong2022robust} also ignore safety constraints in their behavior policy, leading to potential risks in executing the data-collecting behavior policy. \textit{In contrast, our method incorporates safety constraints into the search for the optimal variance-reducing behavior policy, ensuring safety throughout the execution.} 

Using the per-decision importance sampling estimator, \citet{mukherjee2022revar} also propose a behavior policy to reduce variance. Nevertheless, they also do not consider the crucial safety problem. Besides, their results are restricted to tree-structured MDPs, which is a significant limitation.
\textit{In contrast, our method is applicable on general CMDPs, extensions of the widely studied MDPs.} Moreover, \citet{mukherjee2022revar} also leverage a model-based approach.
The current best method in behavior policy design is proposed by \citet{liu2024efficient}, where they find an optimal and an offline-learnable behavior policy.
However, all these approaches \citep{hanna2017data, zhong2022robust, liu2024efficient} focus solely on reducing evaluation variance while neglecting a critical issue: safety. Without integrating safety constraints into the design of the behavior policy, its online execution could cause unforeseen and severe damage.
\textit{In contrast, we consider the variance minimization problem with safety constraints, obtaining a behavior policy that is safe throughout the execution, while simultaneously achieving substantial variance reduction compared with the on-policy method.}

\section{Background}
A finite Markov decision process (MDP, \citet{puterman2014markov}) includes a finite state space $\fS$, a finite action space $\fA$, 
a reward function $r: \fS \times \fA \to \R$,
a transition probability function $p: \fS \times \fS \times \fA \to [0, 1]$,
an initial state distribution $p_0: \fS \to [0, 1]$,
and a constant horizon length $T$. In this paper, to impose safety constraints, we consider constrained Markov decision processes (CMDPs, \citet{altman2021constrained}), which extends the MDPs with a cost function $c: \fS \times \fA \to [0,\infty)$. 
Without loss of generality, 
we consider the undiscounted setting for simplicity.
Our method is compatible with the discounted setting \citep{puterman2014markov} as long as the horizon is fixed and finite. For any integer, 
we define a shorthand $[n] \doteq \qty{0, 1, \dots, n}$. For any set $\mathcal{X}$, we use $|\mathcal{X}|$ to denote its cardinality. We use $\Delta^{|\mathcal{X}| - 1}$ to denote the $(|\mathcal{X}| - 1)$-dimensional probability simplex, 
representing the set of all probability distributions over the set $\mathcal{X}$.

A constrained Markov decision process (CMDP) begins at time step $0$, where
an initial state $S_0$ is sampled from $p_0$.
At each time step $t \in [T-1]$,
an action $A_t$ is sampled based on $\pi_t(\cdot \mid S_t)$. Here, $\pi_t: \fA \times \fS \to [0, 1]$ denotes the policy at time step $t$. 
Thereafter, a reward $R_{t+1} \doteq r(S_t, A_t)$ and a  cost $C_{t+1}\doteq c(S_t,A_t)$ is emitted by the environment. A successor state $S_{t+1}$ is then sampled from $p(\cdot \mid S_t, A_t)$. 
We define the abbreviation $\pi_{i:j} \doteq \qty{\pi_i, \pi_{i+1}, \dots, \pi_j}$ and $\pi \doteq \pi_{0:T-1}$.
At each time step $t$, the return for the reward $r$ is defined as 
$
  G_t \doteq \sum_{i={t+1}}^T R_i,
$
and the return for the cost $c$ is $G^c_t\doteq  \sum^T_{i={t+1}} C_i$. Then, we define the state-value and action-value functions for the reward $r$ as
$
v_{\pi, t}(s) \doteq \E_{\pi}\left[G_t \mid S_t = s\right]$ and $
q_{\pi, t}(s, a) \doteq \E_{\pi}\left[G_t \mid S_t = s, A_t = a\right].
$
Similarly, the state-value and action-value functions for the cost $c$ are defined as
$
v^c_{\pi, t}(s) \doteq \E_{\pi}\left[G^c_t \mid S_t = s\right]$ and $
q^c_{\pi, t}(s, a) \doteq \E_{\pi}\left[G^c_t \mid S_t = s, A_t = a\right].
$
We adopt the total rewards performance metric from \citet{puterman2014markov} to measure the performance of the policy $\pi$,
which is defined as 
% \begin{align}
$J(\pi) \doteq \sum_s p_0(s) v_{\pi, 0}(s)$.
Likewise, we also define the total costs of $\pi$ as
$J^c(\pi)\doteq \sum_s p_0(s) v_{\pi, 0}^c(s)$.
In this paper, we use Monte Carlo
methods, as introduced by \citet{kakutani1945markoff}, for estimating the total rewards $ J(\pi)$. 
The most straightforward and prevalent technique among many of its variants is to draw samples of $J(\pi)$ through the online execution of the policy $\pi$.
As the number of samples increases, the empirical average of the sampled returns is guaranteed to converge to $J(\pi)$.
This method is called on-policy learning (\citealt{sutton1988learning}) as it estimates a policy $\pi$ by executing itself, 

In this work, we focus on off-policy evaluation, in which the goal is to estimate the total rewards $J(\pi)$ of an interested policy $\pi$, called the \textit{target policy} 
by executing a different policy $\mu$,
called the \textit{behavior policy}. 
We generate each trajectory
$
\textstyle \qty{S_0, A_0, R_1, C_1, S_1, A_1, R_2, C_2, \dots, S_{T-1}, A_{T-1}, R_T, C_T}
$ by a behavior policy $\mu$ with
$
  S_0 \sim p_0, A_{t} \sim \mu_{t}(\cdot | S_{t}).
$
For simplicity, we use a shorthand $\tau^{\mu_{t:T-1}}_{t:T-1}$ for a trajectory generated by the behavior policy $\mu$ from the time step $t$ to the time step $T-1$ inclusively. It is defined as
$
\tau^{\mu_{t:T-1}}_{t:T-1} \doteq \qty{S_t, A_t, R_{t+1}, C_{t+1}\dots, S_{T-1}, A_{T-1}, R_{T}, C_{T}}.
$
In off-policy evaluation, to give an estimate of $J(\pi)$, we adopt the importance sampling ratio to reweigh rewards collected by the behavior policy $\mu$.
We define the importance sampling ratio at time $t$ as
$
\textstyle \rho_t \doteq \frac{\pi_t(A_t \mid S_t)}{\mu_t(A_t \mid S_t)}.
$
We also define the product of importance sampling ratios from time $t$ to $t' \geq t$ as 
$
\textstyle \rho_{t:t'} \doteq \prod_{k=t}^{t'} \frac{\pi_k(A_k | S_k)}{\mu_k(A_k | S_k)}.
$
Various methods utilize importance sampling ratios within off-policy learning frameworks \citep{geweke1988antithetic, hesterberg1995weighted, koller2009probabilistic,thomas2015safe}.
In this paper,
We study the per-decision importance sampling estimator (PDIS, \citet{precup:2000:eto:645529.658134}).
The PDIS Monte Carlo estimator is defined as $\textstyle \pdisg(\tau^{\mu_{t:T-1}}_{t:T-1}) \doteq \sum_{k=t}^{T-1} \rho_{t:k} R_{k+1}.$
We also use the recursive expression of the PDIS estimator as
\begin{align}\label{eq:PDIS-recursive}
\pdisg(\tau^{\mu_{t:T-1}}_{t:T-1}) 
=&\begin{cases}
\rho_t \left(R_{t+1} + \pdisg(\tau^{\mu_{t+1:T-1}}_{t+1:T-1})\right) & t \in [T-2], \\
\rho_tR_{t+1} & t = T-1.
\end{cases}
\end{align}
With the classic policy coverage assumption \citep{precup:2000:eto:645529.658134,maei2011gradient,sutton2016emphatic,zhang2022thesis,liu2024ode} 
$\forall t, s, a, \quad\mu_t(a|s) = 0 \implies \pi_t(a|s) = 0, $
% \begin{align}
% \forall t, s, a, \quad\mu_t(a|s) = 0 \implies \pi_t(a|s) = 0,    
% \end{align}
$\pdisg$ provides an \emph{unbiased} estimation  for $J(\pi)$, i.e., $\E\qty[\pdisg(\tau^{\mu_{0:T-1}}_{0:T-1})] = J(\pi).   $
% \begin{align}
% \E\qty[\pdisg(\tau^{\mu_{0:T-1}}_{0:T-1})] = J(\pi).    
% \end{align}
Since the PDIS estimator is unbiased,
reducing its variance is sufficient for improving its sample efficiency.
We achieve this variance reduction by designing and learning proper behavior policies.

\section{Constrained Variance Minimization for Contextual Bandits}
\label{section: bandit}
In this section, we focus on variance minimization in policy evaluation under safety constraints in contextual bandits. These discussions provide the foundation for the more complicated optimization problems in sequential reinforcement learning settings, which we explore in Section~\ref{section: sequential}. Notations defined in this section are independent of the rest of the paper.

We consider contextual bandits as one-step CMDPs, where the trajectories are in the form of $(s,a,r,c)$. 
% Since the horizon is $1$ in contextual bandits, we omit the subscript $t$ in $\pi$, $\mu$, $S$, and $A$. 
To estimate the performance of the target policy $\pi$, $\E_{a\sim \pi}[r(s,a)]$, with data collected by a behavior policy $\mu$, we adopt the importance sampling ratio \citep{Rubinstein1981Simulation} to reweigh the reward collected by $\mu$. That is, we use $\E_{a\sim \mu}[\rho(a|s)r(s,a)]$ as an estimator, where $\rho(a|s)=\frac{\pi(a|s)}{\mu(a|s)}$. 
Recall $\Delta^{|\mathcal{A}| - 1}$ is the probability simplex representing all probability distributions over the set $\mathcal{A}$.
To ensure that this off-policy evaluation is unbiased,
a classic choice by \citet{Rubinstein1981Simulation} searches for $\mu$ in
\begin{align}
\Lambda_- \doteq \qty{\mu \mid \forall s, a, \mu(a|s) = 0 \Rightarrow \pi(a|s) = 0 \wedge  \forall s, \mu(\cdot | s) \in \Delta^{|\fA| - 1}}.  \label{eq stats search space small}
\end{align}
In this work, we search in an enlarged space $\Lambda$ \citep{owen2013monte, liu2024efficient}, where 
\begin{align}
 \Lambda \doteq \qty{\mu  \mid \forall s,a, \mu(a|s) = 0 \Rightarrow \pi(a|s)r(s,a) = 0 \wedge  \forall s, \mu(\cdot | s) \in \Delta^{|\fA| - 1}}.
  \label{eq: stats search space}
\end{align} 
Although a behavior policy $\mu$ in $\Lambda$ may not cover the target policy $\pi$,  $\mu$ still gives unbiased estimation in statistics. In the following lemma, we show that searching for $\mu$ in this enlarged space $\Lambda$ guarantees unbiasedness in the contextual bandits setting.
\begin{lemma}
\label{lemma: unbias bandits}
    $\forall \mu \in \Lambda$, $\forall s$,
    \begin{align}
    \E_{a\sim \mu}[\rho(a|s)r(s,a)] =\E_{a\sim \pi}[r(s,a)].   
    \end{align}
\end{lemma}
Its proof is in Appendix~\ref{appendix: unbias bandits}. Our goal is to search for a variance-minimizing behavior policy $\mu$. Except for the unbiasedness guaranteed by the search space $\Lambda$, we also require $\mu$ to satisfy safety constraints which will be defined later. 
We formulate the variance minimization objective as, $\forall s$,
\begin{align}
\min_{\mu\in \Lambda} \quad \V_{a\sim \mu}(\rho(a|s)r(s,a)).\label{eq: variance object plain bandits}
\end{align}

Then, with the unbiasedness in Lemma~\ref{lemma: unbias bandits}, we can further decompose the objective in \eqref{eq: variance object plain bandits} as
\begin{align}
\label{eq: bandits obective equation}
\V_{a\sim\mu}(\rho(a|s)r(s,a))=&\E_{a\sim\mu}[(\rho(a|s)r(s,a))^2]-\E_{a\sim\mu}[\rho(a|s)r(s,a)]^2\\
=&\E_{a\sim\mu}[\rho(a|s)^2r(s,a)^2]-\E_{a\sim\pi}[r(s,a)]^2. \explain{By Lemma~\ref{lemma: unbias bandits}}
\end{align}
Since the second term is a constant and is unrelated to $\mu$, it suffices to solve
\begin{align}
\min_{\mu\in \Lambda} \quad & \E_{a\sim\mu}[\rho(a|s)^2r(s,a)^2].\label{eq: bandits objective modified}
\end{align}
Next, to ensure the safety of executing the behavior policy $\mu$, we incorporate a safety constraint into the variance minimization problem. Since measuring safety by the expected cost is a common approach in the safety RL community \citep{berkenkamp2017safe, achiam2017constrained, chow2018lyapunov}, we require that the expected cost of $\mu$ remains within a threshold related to the expected cost of $\pi$. 
Given a safety parameter $\epsilon\in [0,\infty)$, define a cost threshold
\begin{align}
\delta_\epsilon(s)\doteq(1+\epsilon)\E_{a\sim \pi}[c(s,a)].
\end{align}
We impose the following constraint to the optimization problem  \eqref{eq: bandits objective modified}
\begin{align}
    \E_{a\sim \mu}[c(s,a)]\leq \delta_\epsilon(s), \quad \forall s. \label{eq: safety constraint bandits}
\end{align}
This constraint requires that the expected cost of the designed behavior policy $\mu$ should be smaller than the multiple of the expected cost of the target policy $\pi$. 
By satisfying this constraint, we maintain a desired level of safety during the execution of the behavior policy $\mu$. This safety is defined with respect to the target policy $\pi$, which is executed 
in the classic on-policy evaluation method.
By setting $\epsilon = 0$, behavior policies satisfying this constraint are guaranteed to be safer than the target policy.

Notably, another line of research focused on policy safety chooses a constant threshold for the expected cost. We can simply modify \eqref{eq: safety constraint bandits} into a constant-threshold constraint by replacing the threshold function $\delta_\epsilon (s)$ with a constant $\delta$. 
However, such absolute thresholds may make optimization problems infeasible. Strong assumptions on environments and policies have to be made to guarantee the existence of feasible solutions under absolute threshold \citep{achiam2017constrained}. 
In this paper, we impose the safety constraint with respect to the target policy $\pi$, because our goal is to design a safe behavior policy to address the high variance associated with classic on-policy evaluation methods.
The parameter $\epsilon$ in our threshold allows RL practitioners to adjust safety tolerance based on the specific requirements of the problem, as safety constraints are often highly problem-dependent \citep{achiam2017constrained}. In Section~\ref{sec: experiment}, we demonstrate our method in sequential reinforcement learning with a harsh threshold, $\epsilon=0$, achieving both variance and cost reduction compared to the on-policy method.

We formally define our optimization problem and prove its convexity and feasibility in the following theorem.
\begin{lemma}
\label{lemma: convex bandits}
For all $\epsilon$ and $s$, the following optimization problem is convex and feasible. 
% \lsz{try not add too}
\begin{align}
\label{eq: optimization object bandits}
\min_{\mu\in \Lambda} \quad& \E_{a\sim\mu}[\rho(a|s)^2r(s,a)^2],\\
\text{s.t.} \quad&\E_{a\sim \mu}[c(s,a)]\leq \delta_\epsilon(s). \label{eq: safety constraint bandits appendix}
\end{align}
\end{lemma}
Its proof is in Appendix~\ref{appendix: convex bandits}.
Use $\mu^*$ to denote the optimal solution of the above optimization problem.  We have the following lemma.
\begin{lemma}
\label{lemma: optimal bandits}
For all $ \epsilon$ and $s$,  let $\mu^*$ be the optimal solution of optimization problem \eqref{eq: optimization object bandits},  we have 
\begin{align}
\V_{a\sim \mu^*}(\rho(a|s)r(s,a))\leq \V_{a\sim \pi}(r(s,a)).
\end{align}
\end{lemma}
\begin{proof}
We first show that the target policy $\pi$ is always in the feasible set of the optimization problem \eqref{eq: optimization object bandits}. We define the set of feasible policies as 
\begin{align}
\textstyle\mathcal{F}\doteq\{ \mu\in \Lambda \mid \forall \epsilon,s, \E_{a\sim \mu}[c(s,a)]\leq \delta_\epsilon(s)\}.\label{eq: feasible set bandits}
\end{align}
Because $\epsilon \in [0, \infty)$, for the safety constraint, we have
\begin{align}
\E_{a\sim \pi}[c(s,a)] \leq (1+\epsilon) \E_{a\sim \pi}[c(s,a)] =  \delta_\epsilon(s).
\end{align}
By the definition of $\Lambda$ \eqref{eq: stats search space}, $\pi\in \Lambda$. 
Thus, $\pi\in\mathcal{F}$.
Because 
\begin{align}
\mu^* \doteq&\argmin_{\mu\in \mathcal{F}} \textstyle \E_{a\sim\mu}[\rho(a|s)^2r(s,a)^2] \label{eq: mu star argmin}
\end{align}
is the optimal solution, we have
\begin{align}
&\V_{a \sim \mu^*}(\rho(a|s)r(s,a)) \\
=& \E_{a\sim\mu^*}[\rho(a|s)^2r(s,a)^2]-\E_{a\sim\pi}[r(s,a)]^2 \explain{by \eqref{eq: bandits obective equation}} \\
\leq& \E_{a\sim\pi}[\rho(a|s)^2r(s,a)^2]-\E_{a\sim\pi}[r(s,a)]^2  \explain{by \eqref{eq: mu star argmin}} \\
=& \textstyle \E_{a\sim\pi}[r(s,a)^2] -\E_{a\sim\pi}[r(s,a)]^2  \\
=& \V_{a \sim \pi}(r(s,a)) . 
\end{align}
\end{proof}
In Section~\ref{section: sequential}, we expand Lemma~\ref{lemma: convex bandits} and Lemma~\ref{lemma: optimal bandits} from contextual bandits to sequential reinforcement learning in Theorem \ref{theorem: convex rl} and Theorem \ref{theorem: optimal rl}. We show that with a recursive expression of the estimation variance, we can reduce the sequential problem into bandits in each time step $t$, and thereafter obtain the optimal behavior policy $\mu^*$ that minimizes variance under safety constraints.

\section{Constrained Variance Minimization for Sequential Reinforcement Learning}
\label{section: sequential}
In this section, we extend the techniques from contextual bandits to the sequential reinforcement learning setting. We seek to find an optimal behavior policy $\mu$ that reduces the variance $\V\left(\pdisg(\tau^{\mu_{0:T-1}}_{0:T-1})\right)$ under safety constraints. Before defining the optimization problem, we first define the policy space we search for the behavior policy to ensure the unbiasedness of the PDIS estimator. Conventional methods search $\mu$ in the set of all policies that cover the target policy $\pi$ \citep{sutton2018reinforcement}, i.e.,
\begin{align}
\Lambda_- \doteq 
\{& \mu \mid
\forall t, s, a, \mu_t(a|s) = 0 \Rightarrow\pi_t(a|s) = 0\wedge  \forall t,s, \mu_t(\cdot | s) \in \Delta^{|\fA| - 1}\}.
\end{align}
In this paper, similar to the bandits setting \eqref{eq: stats search space}, we search in an enlarged set $\Lambda$, which is defined as
\begin{align}
\Lambda \doteq& \{\mu \mid \forall t, s, a, \mu_t(a|s) = 0 \Rightarrow 
  \pi_t(a|s)q_{\pi, t}(s, a) = 0 \wedge   \forall t,s, \mu_t(\cdot | s) \in \Delta^{|\fA| - 1} \}.\label{eq: definition lambda rl}
\end{align}
The following lemma from \citet{liu2024efficient} ensures the unbiasedness of the off-policy estimator with the behavior policy $\mu \in \Lambda$.
\begin{lemma}
\label{lem rl pdis unbaised}
$\forall \mu \in \Lambda$, $\forall  t$, $\forall  s$, 
\begin{align}
    \E\left[\pdisg(\tau^{\mu_{t:T-1}}_{t:T-1}) \mid S_t = s\right] = v_{\pi, t}(s).
\end{align}
\end{lemma}
Its proof is in Appendix~\ref{appendix: rl unbias}. 
A natural idea to do variance minimization under safety constraints with a safety parameter $\epsilon \in [0, \infty)$ is to solve the following optimization problem
\begin{align}
\label{eq rl opt problem naive}
\min_{\mu \in \Lambda} \quad & \V\left(\pdisg(\tau^{\mu_{0:T-1}}_{0:T-1})\right), \\
\text{s.t.} \quad&  
J^c(\mu)\leq (1+\epsilon)J^c(\pi),
\end{align}
where $J^c(\mu)\doteq \sum_s p_0(s) v_{\mu, 0}^c(s)$ is the expected cost of the behavior policy $\mu$. Solving this problem directly is very challenging. When designing a policy at a time step $t$, we need to consider not only the immediate reward generated by this action but also the future consequences. \cite{hanna2017data} try to solve this problem without safety constraints by directly optimizing the behavior policy $\mu$ with gradient descent. However, this approach requires online data to optimize $\mu$ and struggles in even moderately complicated environments as shown in \cite{zhong2022robust} and \citet{liu2024efficient}. 

In this paper, we therefore propose to solve this problem in a backward way while ensuring safety constraints.
Given an $\epsilon$, 
use 
\begin{align}
\label{eq: constraint j}
\delta_{\epsilon, t}(s)\doteq (1+\epsilon)v^c_{\pi,t}(s)
\end{align}
to denote the safety threshold.
We define an extended reward function $\tilde{r}_t(s,a)$ and a behavior policy $\mu^*$. 
They are defined in the order of $ \qty{\tilde{r}_{T-1}, \mu^*_{T-1}, \tilde{r}_{T-2},\mu^*_{T-2}, \cdots , \tilde{r}_{0},\mu^*_{0}}$. Denote the variance of the state value for the next state
% i.e.,
% $v_{\pi, t+1}(S_{t+1})$,
given the current state-action pair $(s,a)$ as $\nu_{\pi, t}(s, a)$.
We have
\begin{align}
\textstyle
\nu_{\pi, t}(s, a)
\doteq&
\begin{cases}
 0 & t = T-1,\\
\V_{S_{t+1}}\left(v_{\pi, t+1}(S_{t+1})\mid S_t=s, A_t=a\right)  & t \in [T-2].
\end{cases}
\label{def:nu}
\end{align}
Then, the extended reward function is defined as
\begin{align}
\textstyle
\tilde{r}_t(s,a) 
\doteq&
\begin{cases}
r_{\pi,t}(s,a)^2 & t = T-1,\\
\nu_{\pi,t}(s,a)
+ q_{\pi, t}(s,a)^2+\E_{S_{t+1}}\left[\V\left(\pdisg(\tau^{\mu^*_{t+1:T-1}}_{t+1:T-1})\mid S_{t+1}\right) \mid s,a\right]  & t \in [T-2].
\end{cases}
\label{eq: extended reward}
\end{align}
The behavior policy $\mu_t^*$ is defined as the optimal solution to the following problem. $\forall t,s$,
\begin{align}
\min_{\mu_t\in \Lambda} \quad& \E_{a\sim\mu_t}[\rho_t^2\tilde{r}_t(s,a)],\\
\text{s.t.} \quad&\E_{a\sim \mu_t}[q^c_{\mu,t}(s,a)]\leq \delta_{\epsilon, t}(s). 
\end{align}
We have the following theorem showing the convexity and feasibility of \eqref{eq: optimization object rl}, thus ensuring the existence of the behavior policy $\mu^*$.
\begin{theorem}
\label{theorem: convex rl}
$\forall \epsilon \geq 0$, $\forall t$, $\forall s$, the following optimization problem is convex and feasible.
\begin{align}
\min_{\mu_t\in \Lambda} \quad& \E_{a\sim\mu_t}[\rho_t^2\tilde{r}_t(s,a)],\label{eq: optimization object rl}\\
\text{s.t.} \quad&\E_{a\sim \mu_t}[q^c_{\mu,t}(s,a)]\leq \delta_{\epsilon, t}(s). \label{eq: safety constraint rl appendix}
\end{align}
\end{theorem}
Its proof is in Appendix~\ref{appendix: convex rl}. We notice that the constrained optimization problem \eqref{eq: optimization object rl} is similar to \eqref{eq: optimization object bandits}, which is the optimization problem introduced in Section \ref{section: bandit}. 
In the contextual bandit setting \eqref{eq: optimization object bandits}, we optimize the objective with respect to the reward function $r$, ensuring variance reduction (Lemma~\ref{lemma: optimal bandits}).
In sequential reinforcement learning \eqref{eq: optimization object rl}, we optimize with respect to the extended reward function $\tilde{r}$, achieving variance reduction (Theorem~\ref{theorem: optimal rl} and \eqref{eq: trajectory variance}), while simultaneously guaranteeing safety \eqref{eq: trajectory safety}.
This observation provides a key insight: the step-wise optimization problem in \textit{sequential reinforcement learning} can be viewed as a reduced optimization problem in one-step \textit{contextual bandits}, where the reward is 
$\tilde{r}$. 
In Section \ref{section: learn behavior policy}, we further propose an efficient algorithm to learn 
$\tilde{r}$
without directly addressing the complicated trajectory variance $\V\left(\pdisg(\tau^{\mu_{t+1:T-1}}_{t+1:T-1})\mid S_{t+1}\right)$, making long-horizon RL problems more tractable.

\begin{restatable}{theorem}
{restatestepwiseoptimal}
\label{theorem: optimal rl}
The behavior policy $\mu^*$ reduces variance compared with the on-policy evaluation method. 
\begin{align}
\textstyle \forall t,s, 
\V\left(\pdisg(\tau^{\mu^*_{t:T-1}}_{t:T-1})\mid S_t=s\right)\leq \V\left(\pdisg(\tau^{\pi_{t:T-1}}_{t:T-1})\mid S_t=s\right).    
\end{align}
\end{restatable}

Its proof is in Appendix~\ref{appendix: optimal rl}.
We also present the following theorem to demonstrate variance reduction and safety guarantee with respect to the original constrained optimization problem \eqref{eq rl opt problem naive}.
\begin{theorem}
\label{theorem: optimal rl final}
For all $\epsilon \geq 0$, the corresponding behavior policy $\mu^*$ has the following property
\begin{align}
&\text{1.}\quad\textstyle 
\V\left(\pdisg(\tau^{{\mu}^*_{0:T-1}}_{0:T-1})\right) \leq \V\left(\pdisg(\tau^{\pi_{0:T-1}}_{0:T-1})\right)\label{eq: trajectory variance}\\
&\text{2.}\quad J^c(\mu^*)\leq (1+\epsilon)J^{c}(\pi)\label{eq: trajectory safety}
\end{align}
\end{theorem}
Its proof is in Appendix~\ref{appendix: optimal rl final}. Notably, \eqref{eq: trajectory safety} shows that our step-wise safety-constraint \eqref{eq: safety constraint rl appendix} 
is stricter than the original constraint \eqref{eq rl opt problem naive}.

\begin{algorithm}[ht]
\caption{Safety-Constrained Optimal Policy Evaluation (SCOPE)}
\label{alg: safe algorithm}
\begin{algorithmic}[1]
\STATE {\bfseries Input:} 
a target policy $\pi$, \\
\hspace{1cm} an offline dataset $\mathcal{D} = \qty{(t_i,s_i,a_i,r_i, c_i, s'_i)}_{i=1}^m$
\STATE {\bfseries Output:} a behavior policy $\mu^*$
\STATE Approximate $q_{\pi,t}, q^c_{\pi,t}$ from $\mathcal{D}$ 
\FOR {$t = T-1$ to $0$}
\STATE Approximate $\tilde{r}_{t}$ from $\mathcal{D}$  by Lemma~\ref{lemma: recursive tilde r}
\STATE Approximate $\mu^*_t(a|s)$ following \eqref{eq: optimization object rl}
\ENDFOR
\STATE \textbf{Return:} 
the approximated behavior policy $\mu^*$
\end{algorithmic}
\end{algorithm}

\section{Learning the Optimal Behavior Policy}
\label{section: learn behavior policy}

In this section, we propose an efficient algorithm to learn $\tilde{r}$ with previously logged offline data, and subsequently derive the optimal behavior policy $\mu^*$ under safety constraints.
We notice that learning $\tilde{r}$ by \eqref{eq: extended reward} is inefficient since we need to approximate the complicated variance $\V\left(\pdisg(\tau^{\mu_{t+1:T-1}}_{t+1:T-1})\mid S_t\right) $, which involves the entire future trajectory. To tackle this challenge, we present a recursive expression of $\tilde{r}$ in the following lemma.

\begin{restatable}[]{lemma}{reOOlemmaOOrecursiveOOtildeOOr}
\label{lemma: recursive tilde r}
$\forall s,a$, when $t=T-1$, $\tilde{r}_t(s,a)=r_{\pi,t}(s,a)^2$. When $t\in [T-2]$,
\begin{align}
\textstyle\tilde{r}_t(s,a)=2q_{\pi,t}(s,a)r(s,a)- r(s,a)^2+\E_{s' \sim p, a' \sim \pi}\qty[\frac{\pi_{t+1}(a'|s')}{\mu^*_{t+1}(a'|s')}\tilde{r}_{\pi,t+1}(s',a')] . 
\end{align}
\end{restatable}

Its proof is in Appendix~\ref{appendix: recursive tilde r}. With this lemma, we can learn $\tilde{r}$ recursively without approximating the complicated trajectory variance. Then, by \eqref{eq: variance r - v} in the appendix, we can 
also decompose the widely interested variance target in a succinct form
\begin{align}
\underbrace{\V\left(\pdisg(\tau^{\mu^*_{t:T-1}}_{t:T-1})\mid S_t=s\right)}_{\hypertarget{a}{(a)}} =  \underbrace{\E_{a\sim\mu}[\rho_t^2\tilde{r}_t(s,a)]}_{\hypertarget{b}{(b)}} - \underbrace{v_{\pi,t}(s)^2}_{\hypertarget{c}{(c)}}, \quad \forall s,t. \label{eq: variance decompose main}
\end{align}
This succinct form offers a way to approximate the complicated 
trajectory 
variance term \hyperlink{a}{(a)} from \hyperlink{b}{(b)} and \hyperlink{c}{(c)}, which do not contain any variance term themselves.
This is a surprising result because previously the best simplification of the variance for off-policy estimator \hyperlink{a}{(a)} still depends on state-value variance terms \citep{jiang2015doubly, liu2024efficient}.
With \eqref{eq: variance decompose main}, we can approximate the variance of the off-policy estimator in a model-free way with only segmented offline data.

For broad applicability, we adopt the behavior policy-agnostic offline learning setting \citep{nachum2019dualdice}, where the offline data has $m$ previously logged data tuples in the form of $
\qty{(t_i, s_i, a_i, r_i, c_i, s_i')}_{i=1}^m$. These data tuples can be generated by one or more possibly unknown behavior policies, and they are not required to form a complete trajectory. In the $i$-th data tuple, $t_i$ is the time step, $s_i$ is the state at time step $t_i$, $a_i$ is the action taken, $r_i$ is the observed reward, $c_i$ is the observed cost, and $s_i'$ is the successor state. 
In this paper, we learn $\tilde{r}$ from previously logged offline data.
Previously logged offline data are cheap and readily available compared with online data. This makes them a great engine for improving policy evaluation in the online phase.
Compared with gradient-based methods \citep{hanna2017data, zhong2022robust} which need complete online trajectories, our method does not require a long online warm-up time to find a good behavior policy because we are able to utilize offline data.
Subsequently, the optimal variance-reducing behavior policy $\mu^*$ under safety constraints is approximated through standard convex optimization solvers \citep{nocedal1999numerical, agrawal2018rewriting}.

\section{Empirical Results}\label{sec: experiment}
In this section, we demonstrate the empirical results comparing our methods against three baselines: \tb{(1)} the on-policy Monte Carlo estimator, 
\tb{(2)} the robust on-policy sampling estimator (ROS, \citet{zhong2022robust}), and 
\tb{(3)} the offline data informed estimator (ODI, \citet{liu2024efficient}). To ensure our method attains lower cost and is thus even safer than the on-policy estimator, we choose $\epsilon=0$ in the threshold $\delta_{\epsilon,t}$ \eqref{eq: constraint j}.
All methods 
learn their required parameters from the same offline dataset to ensure fair comparisons.
Given previously logged offline data, our method learns the optimal behavior policy 
under safety constraints 
using Algorithm~\ref{alg: safe algorithm}.

We name our algorithm Safety-Constrained Optimal Policy Evaluation (SCOPE) to emphasize that safety constraints are inherently considered in the design of the variance-minimizing behavior policy, unlike previous methods that overlook safety concerns. A metaphor for SCOPE is that it builds a bridge focused on efficient transportation (evaluation efficiency) while simultaneously ensuring traffic safety (satisfying safety constraints).

\begin{figure}[t]
\begin{minipage}{0.49\textwidth}
\centering
\includegraphics[width=1\textwidth]{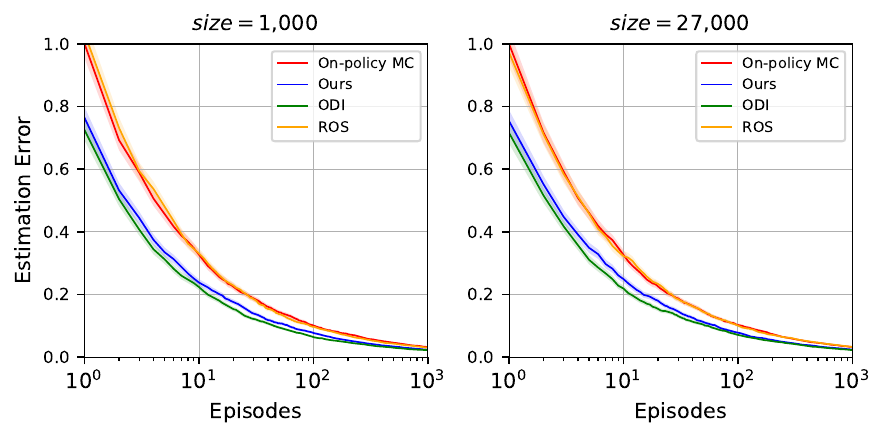}
\caption{Results on Gridworld with \textit{episodes} as x-axis. 
Each curve is averaged over 900 runs (30 target policies, each having 30 independent runs).
Shaded regions denote standard errors and are invisible for some curves as they are too small.
}
\label{fig:gridworld}
\end{minipage}
\hfill
\begin{minipage}{0.49\textwidth}
\centering
\includegraphics[width=1\textwidth]{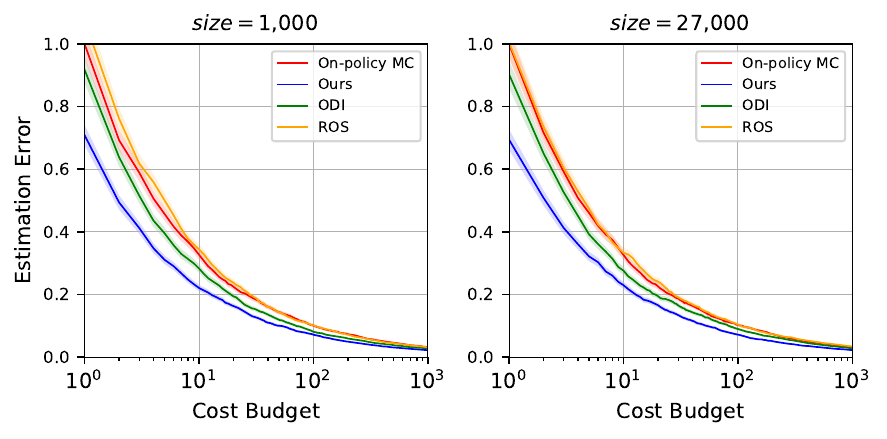}
\caption{Results on Gridworld with \textit{cost budget} as x-axis.  \textit{Cost budget} is
the total cost of execution.
Each curve is averaged over 900 runs (30 target policies, each having 30 independent runs).
Shaded regions denote standard errors.
}
\label{fig:gridworld_cost}
\end{minipage}
\end{figure}

\begin{table}[t]
\centering
\begin{tabular}{llllllll}
\toprule
Env Size& On-policy MC & Ours   & ODI & ROS \\
\midrule
1,000 & 1.000 & \textbf{0.861} & 1.602 & 1.083 \\
27,000 & 1.000 & \textbf{0.849} & 1.590 & 1.067 \\
\bottomrule
\end{tabular}
\caption{Average trajectory cost on Gridworld. Numbers are normalized by the cost of the on-policy estimator. ODI and ROS have much larger costs because they both ignore safety constraints. \textbf{Our method is the only method achieving both variance reduction and constraint satisfaction.}
}
\label{table: gridworld cost}
\end{table}

\textbf{Gridworld:} 
We first conduct experiments in Gridworld  with $n^3$ states. Each Gridworld is an $n \times n$ grid with the time horizon also being $n$. Gridworld environments offer a great tool to test algorithm scalability, because the number of states scales cubically with $n$.  Gridworld in our experiments have $n^3 = 1,000$ and $n^3 = 27,000$ number of states, which are the largest Gridworld environments tested among related works \citep{zhong2022robust, liu2024efficient}.
We test all methods on target policies with various performances.
The offline data is generated by many different policies to simulate previously logged offline data.
In Figure~\ref{fig:gridworld}, we report the estimation error against episodes. 
The estimation error for any line is the absolute error normalized by the absolute error of the on-policy estimator after the first episode. 
Thus, the estimation error of the on-policy estimator starts from $1$. In Figure~\ref{fig:gridworld_cost}, we report the estimation error against 
the total cost of execution.

If considering \textit{solely} variance reduction, Figure~\ref{fig:gridworld} shows our method outperforms the on-policy estimator and ROS by a large margin. 
Admittedly, ODI \citep{liu2024efficient} is slightly better than our method in terms of variance reduction. However, this slight advantage comes with a huge \textit{trade-off} of safety. As shown in Table \ref{table: gridworld cost}, ODI has a much larger cost than on-policy evaluation method (more than $1.5$ times) and our method (almost twice as much). \textbf{This addresses the underestimated fact---solely reducing variance without safety constraints leads to high-cost (unsafe) methods.}

To further demonstrate the superiority of balancing variance reduction and safety cost of our method, we provide Figure~\ref{fig:gridworld_cost} to compare the variance reduction each method achieves with the same cost budget. Since our method SCOPE is optimal for safety-constrained variance minimization, it consistently outperforms all baselines in Figure~\ref{fig:gridworld_cost}, as shown by the lowest blue line. This means that compared with existing best-performing methods, SCOPE needs less cost to achieve the same level of accuracy. 
From Figure~\ref{fig:gridworld_cost}, we compute that to achieve the same accuracy that the on-policy estimator achieves with $1000$ costs (each on-policy episode has expected cost $1$ by normalization), ODI costs $880$ and SCOPE costs only $425$. Following this computation, our method saves $57.5\%$ of costs compared to the on-policy method, and $50\%$ compared to ODI.
\textbf{This reinforces the underestimated fact from the opposite direction---ensuring safety constraints along with the variance minimization leads to a low-cost method.} Also, notably, our estimator outperforms the on-policy and ROS estimators in \textit{reducing both variance and cost}.

\textbf{MuJoCo:} 
Next, we conduct experiments in MuJoCo robot simulation tasks \citep{todorov2012mujoco}.
MuJoCo is a physics engine with a variety of stochastic environments The goal is to control a robot to achieve different behaviors such as walking, jumping, and balancing.

As confirmed in Table \ref{table: mujoco variance} and Table \ref{table: mujoco single cost} in the appendix, our method is the only method consistently achieving both variance reduction and safety constraint satisfaction.
Figure \ref{fig:mujoco} again
indicates that our method consistently outperforms all baselines on reducing variance under the same cost budget. 
This advantage is observed across all five environments, demonstrating the stableness of our method in balancing variance reduction and cost management.
Numerically, in Table~\ref{table: compare num}, we show that our method, SCOPE, saves up to $57.5\%$ cost to achieve the desired evaluation accuracy.
More experiment details are in Appendix~\ref{append: experiment}. It is worth mentioning that our method is robust to hyperparameter choices---all hyperparameters in our method are tuned offline and stay the same across all environments.

\begin{figure}[t]
\includegraphics[width=1\textwidth]{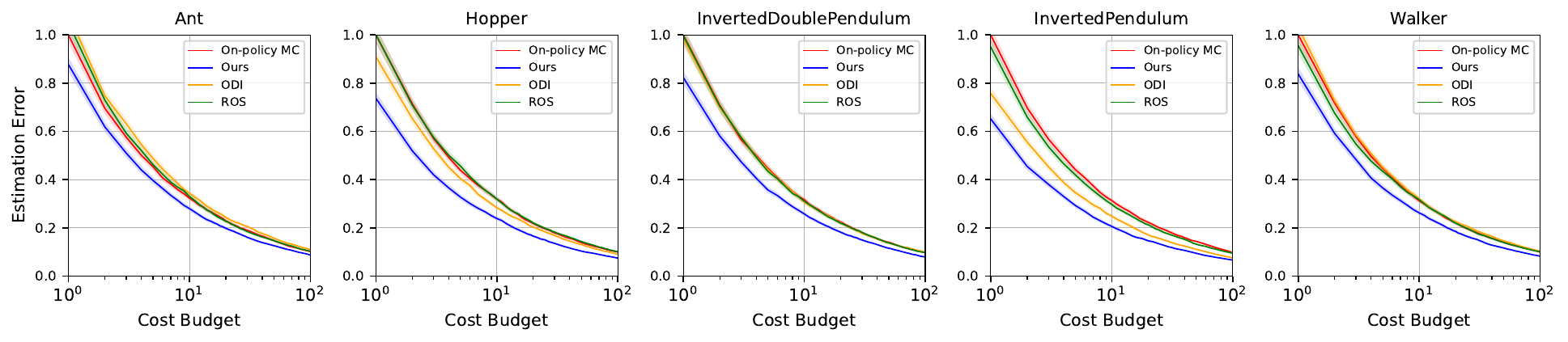}
\centering
\caption{
Results on MuJoCo. \textit{Cost budget} on the x-axis is
the total cost of execution.
Each curve is averaged over 900 runs (30 of target policies, each having 30 independent runs). 
Shaded regions denote standard errors and are invisible for some curves because they are too small. Results with a larger x-axis range are in the appendix. 
}
\label{fig:mujoco}
\end{figure}

\begin{table}
\begin{center}
\begin{small}
\begin{tabular}{llllll}
\toprule
 & On-policy MC & \textbf{Ours} & ODI & ROS & Saved Cost Percentage \\
\midrule
Ant & 1000 & \textbf{746} & 1136 & 1063 & (1000 - 746)/1000 = \textbf{25.4\%} \\
Hopper & 1000 & \textbf{552} & 824 & 1026 & (1000 - 552)/1000 = \textbf{44.8\%} \\
I. D. Pendulum & 1000 & \textbf{681} & 1014 & 1003 & (1000 - 681)/1000 = \textbf{31.9\%} \\
I. Pendulum & 1000 & \textbf{425} & 615 & 890 & (1000 - 425)/1000 = \textbf{57.5\%} \\
Walker & 1000 & \textbf{694} & 1031 & 960 & (1000 - 694)/1000 = \textbf{30.6\%} \\
\bottomrule
\end{tabular}
\end{small}
\end{center}
\caption{Cost needed to achieve the same estimation accuracy that on-policy Monte Carlo achieves with $1000$ episodes on MuJoCo. Each curve is averaged over 900 runs. Standard errors are plotted in Figure \ref{fig:mujoco}.
}
\label{table: compare num}
\end{table}

\section{Conclusion}
In reinforcement learning, due to the sequential nature, policy evaluation often suffers from large variance and thus requires massive data to achieve the desired level of accuracy. In addition, safety is a critical concern for policy execution, since unsafe actions can lead to significant risks and irreversible damage. In this paper, we address these two challenges simultaneously: we propose an optimal variance-minimizing behavior policy under safety constraints.

Theoretically, we show that our estimate is unbiased. Moreover, while simultaneously satisfying safety constraints, our behavior policy is proven to achieve lower variance than the classic on-policy evaluation method (Theorem~\ref{theorem: optimal rl}, Theorem~\ref{theorem: optimal rl final}).
We solve the constrained optimization problem without approximating the complicated trajectory variance (Lemma~\ref{lemma: recursive tilde r}), pointing out a promising direction for addressing long-horizon sequential reinforcement learning challenges.

Empirically, compared with existing best-performing methods, we show our method is the only one that achieves both substantial variance reduction and constraint satisfaction for policy evaluation in sequential reinforcement learning. Moreover, it is even superior to previous methods in both variance reduction and execution safety.

Admittedly, as there is no free lunch, if the offline data size is too small—perhaps containing merely a single data tuple—the learned behavior policy in our method may be inaccurate. In this case, for a safe backup, we recommend the on-policy evaluation method. The future work of our paper is to extend the constrained variance minimization technique to temporal difference learning.

\section*{Acknowledgements}
This work is supported in part by the US National Science Foundation (NSF) under grants III-2128019 and SLES-2331904.

\bibliographystyle{apalike}
\bibliography{bibliography}

\begin{thebibliography}{}

\bibitem[Achiam et~al., 2017]{achiam2017constrained}
Achiam, J., Held, D., Tamar, A., and Abbeel, P. (2017).
\newblock Constrained policy optimization.
\newblock In {\em Proceedings of the International Conference on Machine Learning}.

\bibitem[Agrawal et~al., 2018]{agrawal2018rewriting}
Agrawal, A., Verschueren, R., Diamond, S., and Boyd, S. (2018).
\newblock A rewriting system for convex optimization problems.
\newblock {\em Journal of Control and Decision}.

\bibitem[Altman, 2021]{altman2021constrained}
Altman, E. (2021).
\newblock {\em Constrained Markov decision processes}.
\newblock Routledge.

\bibitem[Berkenkamp et~al., 2017]{berkenkamp2017safe}
Berkenkamp, F., Turchetta, M., Schoellig, A., and Krause, A. (2017).
\newblock Safe model-based reinforcement learning with stability guarantees.
\newblock In {\em Advances in Neural Information Processing Systems}.

\bibitem[Boyd et~al., 2004]{boyd2004convexopt}
Boyd, S., Boyd, S.~P., and Vandenberghe, L. (2004).
\newblock {\em Convex optimization}.
\newblock Cambridge university press.

\bibitem[Brockman et~al., 2016]{openai2016gym}
Brockman, G., Cheung, V., Pettersson, L., Schneider, J., Schulman, J., Tang, J., and Zaremba, W. (2016).
\newblock Open{AI} {G}ym.
\newblock {\em ArXiv Preprint}.

\bibitem[Brunke et~al., 2022]{brunke2022safe}
Brunke, L., Greeff, M., Hall, A.~W., Yuan, Z., Zhou, S., Panerati, J., and Schoellig, A.~P. (2022).
\newblock Safe learning in robotics: From learning-based control to safe reinforcement learning.
\newblock {\em Annual Review of Control, Robotics, and Autonomous Systems}.

\bibitem[Chervonyi et~al., 2022]{chervonyi2022semianalytical}
Chervonyi, Y., Dutta, P., Trochim, P., Voicu, O., Paduraru, C., Qian, C., Karagozler, E., Davis, J.~Q., Chippendale, R., Bajaj, G., et~al. (2022).
\newblock Semi-analytical industrial cooling system model for reinforcement learning.
\newblock {\em ArXiv Preprint}.

\bibitem[Chow et~al., 2018]{chow2018lyapunov}
Chow, Y., Nachum, O., Duenez-Guzman, E., and Ghavamzadeh, M. (2018).
\newblock A lyapunov-based approach to safe reinforcement learning.
\newblock In {\em Advances in Neural Information Processing Systems}.

\bibitem[Chua et~al., 2018]{chua2018deep}
Chua, K., Calandra, R., McAllister, R., and Levine, S. (2018).
\newblock Deep reinforcement learning in a handful of trials using probabilistic dynamics models.
\newblock In {\em Advances in Neural Information Processing Systems}.

\bibitem[Deisenroth and Rasmussen, 2011]{deisenroth2011pilco}
Deisenroth, M.~P. and Rasmussen, C.~E. (2011).
\newblock {PILCO:} {A} model-based and data-efficient approach to policy search.
\newblock In {\em Proceedings of the International Conference on Machine Learning}.

\bibitem[Farahmand and Szepesv{\'a}ri, 2011]{farahmand2011model}
Farahmand, A.-m. and Szepesv{\'a}ri, C. (2011).
\newblock Model selection in reinforcement learning.
\newblock {\em Machine Learning}.

\bibitem[Garc{\i}a and Fern{\'a}ndez, 2015]{garcia2015comprehensive}
Garc{\i}a, J. and Fern{\'a}ndez, F. (2015).
\newblock A comprehensive survey on safe reinforcement learning.
\newblock {\em Journal of Machine Learning Research}.

\bibitem[Geweke, 1988]{geweke1988antithetic}
Geweke, J. (1988).
\newblock Antithetic acceleration of monte carlo integration in bayesian inference.
\newblock {\em Journal of Econometrics}.

\bibitem[Gu et~al., 2022]{gu2022review}
Gu, S., Yang, L., Du, Y., Chen, G., Walter, F., Wang, J., and Knoll, A. (2022).
\newblock A review of safe reinforcement learning: Methods, theory and applications.
\newblock {\em ArXiv Preprint}.

\bibitem[Hanna et~al., 2017]{hanna2017data}
Hanna, J.~P., Thomas, P.~S., Stone, P., and Niekum, S. (2017).
\newblock Data-efficient policy evaluation through behavior policy search.
\newblock In {\em Proceedings of the International Conference on Machine Learning}.

\bibitem[Hesterberg, 1995]{hesterberg1995weighted}
Hesterberg, T. (1995).
\newblock Weighted average importance sampling and defensive mixture distributions.
\newblock {\em Technometrics}.

\bibitem[Huang et~al., 2022]{huang2022cleanrl}
Huang, S., Dossa, R. F.~J., Ye, C., Braga, J., Chakraborty, D., Mehta, K., and Ara{\'u}jo, J.~G. (2022).
\newblock Cleanrl: High-quality single-file implementations of deep reinforcement learning algorithms.
\newblock {\em Journal of Machine Learning Research}.

\bibitem[Jiang and Li, 2016]{jiang2015doubly}
Jiang, N. and Li, L. (2016).
\newblock Doubly robust off-policy value evaluation for reinforcement learning.
\newblock In {\em Proceedings of the International Conference on Machine Learning}.

\bibitem[Kakutani, 1945]{kakutani1945markoff}
Kakutani, S. (1945).
\newblock Markoff process and the dirichlet problem.
\newblock In {\em Proceedings of the Japan Academy}.

\bibitem[Kalashnikov et~al., 2018]{kalashnikov2018scalable}
Kalashnikov, D., Irpan, A., Pastor, P., Ibarz, J., Herzog, A., Jang, E., Quillen, D., Holly, E., Kalakrishnan, M., Vanhoucke, V., et~al. (2018).
\newblock Scalable deep reinforcement learning for vision-based robotic manipulation.
\newblock In {\em Proceedings of Conference on Robot Learning}.

\bibitem[Kingma and Ba, 2015]{kingma2014adam}
Kingma, D.~P. and Ba, J. (2015).
\newblock Adam: {A} method for stochastic optimization.
\newblock In {\em Proceedings of the International Conference on Learning Representations}.

\bibitem[Koller and Friedman, 2009]{koller2009probabilistic}
Koller, D. and Friedman, N. (2009).
\newblock {\em Probabilistic Graphical Models: Principles and Techniques}.
\newblock Mit Press.

\bibitem[Le et~al., 2019]{le2019batch}
Le, H.~M., Voloshin, C., and Yue, Y. (2019).
\newblock Batch policy learning under constraints.
\newblock In {\em Proceedings of the International Conference on Machine Learning}.

\bibitem[Li, 2019]{li2019perspective}
Li, L. (2019).
\newblock A perspective on off-policy evaluation in reinforcement learning.
\newblock {\em Frontiers of Computer Science}.

\bibitem[Liu et~al., 2024a]{liu2024doubly}
Liu, S., Chen, C., and Zhang, S. (2024a).
\newblock Doubly optimal policy evaluation for reinforcement learning.
\newblock {\em ArXiv Preprint}.

\bibitem[Liu et~al., 2024b]{liu2024ode}
Liu, S., Chen, S., and Zhang, S. (2024b).
\newblock The {ODE} method for stochastic approximation and reinforcement learning with markovian noise.
\newblock {\em ArXiv Preprint}.

\bibitem[Liu et~al., 2024c]{liu2024multi}
Liu, S., Chen, Y., and Zhang, S. (2024c).
\newblock Efficient multi-policy evaluation for reinforcement learning.
\newblock {\em ArXiv Preprint}.

\bibitem[Liu and Zhang, 2024]{liu2024efficient}
Liu, S. and Zhang, S. (2024).
\newblock Efficient policy evaluation with offline data informed behavior policy design.
\newblock In {\em Proceedings of the International Conference on Machine Learning}.

\bibitem[Liu et~al., 2021]{liu2021policy}
Liu, Y., Halev, A., and Liu, X. (2021).
\newblock Policy learning with constraints in model-free reinforcement learning: A survey.
\newblock In {\em Proceedings of the International Joint Conference on Artificial Intelligence}.

\bibitem[Maei, 2011]{maei2011gradient}
Maei, H.~R. (2011).
\newblock {\em Gradient temporal-difference learning algorithms}.
\newblock PhD thesis, University of Alberta.

\bibitem[Marivate, 2015]{marivate2015improved}
Marivate, V.~N. (2015).
\newblock {\em Improved empirical methods in reinforcement-learning evaluation}.
\newblock Rutgers The State University of New Jersey, School of Graduate Studies.

\bibitem[Mirhoseini et~al., 2021]{mirhoseini2021graph}
Mirhoseini, A., Goldie, A., Yazgan, M., Jiang, J.~W., Songhori, E., Wang, S., Lee, Y.-J., Johnson, E., Pathak, O., Nazi, A., et~al. (2021).
\newblock A graph placement methodology for fast chip design.
\newblock {\em Nature}.

\bibitem[Moldovan and Abbeel, 2012]{moldovan2012safe}
Moldovan, T.~M. and Abbeel, P. (2012).
\newblock Safe exploration in markov decision processes.
\newblock {\em ArXiv Preprint}.

\bibitem[Mukherjee et~al., 2022]{mukherjee2022revar}
Mukherjee, S., Hanna, J.~P., and Nowak, R.~D. (2022).
\newblock Revar: Strengthening policy evaluation via reduced variance sampling.
\newblock In {\em Proceedings of the Conference in Uncertainty in Artificial Intelligence}.

\bibitem[Nachum et~al., 2019]{nachum2019dualdice}
Nachum, O., Chow, Y., Dai, B., and Li, L. (2019).
\newblock Dualdice: Behavior-agnostic estimation of discounted stationary distribution corrections.
\newblock In {\em Advances in Neural Information Processing Systems}.

\bibitem[Nocedal and Wright, 1999]{nocedal1999numerical}
Nocedal, J. and Wright, S.~J. (1999).
\newblock {\em Numerical optimization}.
\newblock Springer.

\bibitem[Owen, 2013]{owen2013monte}
Owen, A.~B. (2013).
\newblock {\em Monte Carlo theory, methods and examples}.
\newblock Stanford.

\bibitem[Precup et~al., 2000]{precup:2000:eto:645529.658134}
Precup, D., Sutton, R.~S., and Singh, S.~P. (2000).
\newblock Eligibility traces for off-policy policy evaluation.
\newblock In {\em Proceedings of the International Conference on Machine Learning}.

\bibitem[Puterman, 2014]{puterman2014markov}
Puterman, M.~L. (2014).
\newblock {\em Markov decision processes: discrete stochastic dynamic programming}.
\newblock John Wiley \& Sons.

\bibitem[Rubinstein, 1981]{Rubinstein1981Simulation}
Rubinstein, R.~Y. (1981).
\newblock {\em Simulation and the Monte Carlo Method}.
\newblock Wiley.

\bibitem[Schulman et~al., 2017]{schulman2017proximal}
Schulman, J., Wolski, F., Dhariwal, P., Radford, A., and Klimov, O. (2017).
\newblock Proximal policy optimization algorithms.
\newblock {\em ArXiv Preprint}.

\bibitem[Sutton, 1988]{sutton1988learning}
Sutton, R.~S. (1988).
\newblock Learning to predict by the methods of temporal differences.
\newblock {\em Machine Learning}.

\bibitem[Sutton, 1990]{sutton1990integrated}
Sutton, R.~S. (1990).
\newblock Integrated architectures for learning, planning, and reacting based on approximating dynamic programming.
\newblock In {\em Proceedings of the International Conference on Machine Learning}.

\bibitem[Sutton and Barto, 2018]{sutton2018reinforcement}
Sutton, R.~S. and Barto, A.~G. (2018).
\newblock {\em Reinforcement Learning: An Introduction (2nd Edition)}.
\newblock MIT press.

\bibitem[Sutton et~al., 2016]{sutton2016emphatic}
Sutton, R.~S., Mahmood, A.~R., and White, M. (2016).
\newblock An emphatic approach to the problem of off-policy temporal-difference learning.
\newblock {\em Journal of Machine Learning Research}.

\bibitem[Sutton et~al., 2008]{sutton2012dyna}
Sutton, R.~S., Szepesv{\'{a}}ri, C., Geramifard, A., and Bowling, M.~H. (2008).
\newblock Dyna-style planning with linear function approximation and prioritized sweeping.
\newblock In {\em Proceedings of the Conference in Uncertainty in Artificial Intelligence}.

\bibitem[Thomas, 2015]{thomas2015safe}
Thomas, P.~S. (2015).
\newblock {\em Safe reinforcement learning}.
\newblock PhD thesis, University of Massachusetts Amherst.

\bibitem[Todorov et~al., 2012]{todorov2012mujoco}
Todorov, E., Erez, T., and Tassa, Y. (2012).
\newblock Mujoco: {A} physics engine for model-based control.
\newblock In {\em Proceedings of the International Conference on Intelligent Robots and Systems}.

\bibitem[Towers et~al., 2024]{towers2024gymnasium}
Towers, M., Kwiatkowski, A., Terry, J., Balis, J.~U., De~Cola, G., Deleu, T., Goul{\textasciitilde a}o, M., Kallinteris, A., Krimmel, M., KG, A., et~al. (2024).
\newblock Gymnasium: A standard interface for reinforcement learning environments.
\newblock {\em ArXiv Preprint}.

\bibitem[Trinh et~al., 2024]{trinh2024solving}
Trinh, T.~H., Wu, Y., Le, Q.~V., He, H., and Luong, T. (2024).
\newblock Solving olympiad geometry without human demonstrations.
\newblock {\em Nature}.

\bibitem[Vinyals et~al., 2019]{vinyals2019grandmaster}
Vinyals, O., Babuschkin, I., Czarnecki, W.~M., Mathieu, M., Dudzik, A., Chung, J., Choi, D.~H., Powell, R., Ewalds, T., Georgiev, P., Oh, J., Horgan, D., Kroiss, M., Danihelka, I., Huang, A., Sifre, L., Cai, T., Agapiou, J.~P., Jaderberg, M., Vezhnevets, A.~S., Leblond, R., Pohlen, T., Dalibard, V., Budden, D., Sulsky, Y., Molloy, J., Paine, T.~L., G{\"{u}}l{\c{c}}ehre, {\c{C}}., Wang, Z., Pfaff, T., Wu, Y., Ring, R., Yogatama, D., W{\"{u}}nsch, D., McKinney, K., Smith, O., Schaul, T., Lillicrap, T.~P., Kavukcuoglu, K., Hassabis, D., Apps, C., and Silver, D. (2019).
\newblock Grandmaster level in starcraft {II} using multi-agent reinforcement learning.
\newblock {\em Nature}.

\bibitem[Wachi et~al., 2024]{wachi2024survey}
Wachi, A., Shen, X., and Sui, Y. (2024).
\newblock A survey of constraint formulations in safe reinforcement learning.
\newblock {\em ArXiv Preprint}.

\bibitem[Wachi and Sui, 2020]{wachi2020safe}
Wachi, A. and Sui, Y. (2020).
\newblock Safe reinforcement learning in constrained markov decision processes.
\newblock In {\em Proceedings of the International Conference on Machine Learning}.

\bibitem[Williams, 1992]{williams1992simple}
Williams, R.~J. (1992).
\newblock Simple statistical gradient-following algorithms for connectionist reinforcement learning.
\newblock {\em Machine Learning}.

\bibitem[Zhang, 2022]{zhang2022thesis}
Zhang, S. (2022).
\newblock {\em Breaking the deadly triad in reinforcement learning}.
\newblock PhD thesis, University of Oxford.

\bibitem[Zhang, 2023]{Zhang_2023}
Zhang, S. (2023).
\newblock A new challenge in policy evaluation.
\newblock In {\em Proceedings of the AAAI Conference on Artificial Intelligence}.

\bibitem[Zhong et~al., 2022]{zhong2022robust}
Zhong, R., Zhang, D., Sch{\"a}fer, L., Albrecht, S.~V., and Hanna, J.~P. (2022).
\newblock Robust on-policy sampling for data-efficient policy evaluation in reinforcement learning.
\newblock In {\em Advances in Neural Information Processing Systems}.

\end{thebibliography}

%%%%%%%%%%%%%%%%%%%%%%%%%%%%%%%%%%%%%%%%%%%%%%%%%%%%%%%%%%%%%%%%%%%%%%%%%%%%%%%
%%%%%%%%%%%%%%%%%%%%%%%%%%%%%%%%%%%%%%%%%%%%%%%%%%%%%%%%%%%%%%%%%%%%%%%%%%%%%%%
% APPENDIX
%%%%%%%%%%%%%%%%%%%%%%%%%%%%%%%%%%%%%%%%%%%%%%%%%%%%%%%%%%%%%%%%%%%%%%%%%%%%%%%
%%%%%%%%%%%%%%%%%%%%%%%%%%%%%%%%%%%%%%%%%%%%%%%%%%%%%%%%%%%%%%%%%%%%%%%%%%%%%%%
\newpage
\appendix
\onecolumn
% \section{You \emph{can} have an appendix here.}

\section{Proofs}

\subsection{Proof of Lemma~\ref{lemma: unbias bandits}}\label{appendix: unbias bandits}
\begin{proof}
$\forall s$, $\forall\mu\in\Lambda$,
  \begin{align}
\E_{a\sim \mu}[\rho(a|s)r(s,a)]=&\sum_{a\in \{a|\mu(a|s)>0\}}\mu(a|s)\frac{\pi(a|s)}{\mu(a|s)}r(s,a)\\
=&\sum_{a\in \{a|\mu(a|s)>0\}}\pi(a|s)r(s,a)\\
=&\sum_{a\in \{a|\mu(a|s)>0\}}\pi(a|s)r(s,a)+\sum_{a\in \{a|\mu(a|s)=0\}}\pi(a|s)r(s,a) \explain{$\mu\in \Lambda$}\\
=&\sum_{a}\pi(a|s)r(s,a)\\
=&\E_{a\sim \pi}[r(s,a)].
    \end{align}    
\end{proof}
\subsection{Proof of Lemma~\ref{lemma: convex bandits}}\label{appendix: convex bandits}
\begin{proof}
To prove Lemma~\ref{lemma: convex bandits}, we express the objective function as
\begin{align}
\E_{a\sim\mu}[\rho(a|s)^2r(s,a)^2]=\sum_{a\in\{a|\mu(a|s)>0\}}\frac{\pi(a|s)^2r(s,a)^2}{\mu(a|s)}.
\end{align}
To prove the problem is convex, we begin by examining the feasible set of each constraint separately. 
% For simplicity, for all $s$, we use $\mu_s$ to denote the probability distribution over $\mathcal{A}$ given $s$.

In the first constraint of  $\Lambda$ \eqref{eq: stats search space}, 
\begin{align}
    \forall s,a,\mu(a|s)=0\implies \pi(a|s)r(s,a)=0.\label{eq: bandits coverage constraint}
\end{align}
The feasible set of \eqref{eq: bandits coverage constraint} is a linear subspace of $\mathbb{R}^{|\mathcal{A}|}$ defined by a set of linear equations. Thus, this feasible set is convex.

Next, we decompose the other constraint of $\Lambda$ \eqref{eq: stats search space}, $\mu(\cdot|s)\in \Delta^{|\fA| - 1}$ $\forall s$, into two subconstraints:
\begin{align}
\sum_a\mu(a|s)=1,\label{eq: sum 1 bandits appendix}\\
\forall a, \mu(a|s)\geq 0. \label{eq: nonnegative bandits appendix}
\end{align}
For all $s$, the feasible set in \eqref{eq: sum 1 bandits appendix} can be written in the vector form as
\begin{align}
\mathbf{1}^T\overrightarrow{\mu_s}=1,    \label{eq: sum1 appendix linear}
\end{align}
where $\mathbf{1}\in \mathbb{R}^{|\mathcal{A|}}$ is the vector of ones defined as 
\begin{align}
\mathbf{1}\doteq
\begin{bmatrix}
1 \\
\vdots\\
1
\end{bmatrix},
\end{align}
and $\overrightarrow{\mu_s}\in \mathbb{R}^{|\mathcal{A}|}$ is defined as
\begin{align}
\overrightarrow{\mu_s}\doteq
\begin{bmatrix}
\mu(a_1|s) \\
\vdots\\
\mu(a_{|\mathcal{A}|}|s)
\end{bmatrix}.
\end{align}
Since \eqref{eq: sum1 appendix linear} is linear, the constraint \eqref{eq: sum 1 bandits appendix} is affine and thus convex \citep{boyd2004convexopt}.

For all $s$, the feasible set of
\eqref{eq: nonnegative bandits appendix} is the non-negative orthant, defined as
\begin{align}
    \mathbb{R}_+^{|\mathcal{A}|}\doteq \{\mu(\cdot|s)\in\mathbb{R}^{|\mathcal{A}|}\mid \mu(a|s)\geq 0, \forall a\}.
\end{align}
Since the non-negative orthant forms a convex cone and is known to be a convex set \citep{boyd2004convexopt}, we conclude that this constraint's feasible set is convex.

Next, we define the vector of costs for all $s$ as
\begin{align}
\mathbf{c}_s\doteq
\begin{bmatrix}
c(s, a_1) \\
\vdots\\
c(s, a_{|\mathcal{A}|})
\end{bmatrix}.
\end{align}
Then, for all $\epsilon$ and $s$, the safety constraint \eqref{eq: safety constraint bandits appendix} can be rewritten as
\begin{align}
\mathbf{c}_s^\top \overrightarrow{\mu_s}\leq \delta_\epsilon(s),
\end{align}
which is a linear inequality in $\mu$. Thus, its feasible set is in a convex half-space. Because all the constraints are convex, we conclude that the feasible set $\mathcal{F}$ in \eqref{eq: feasible set bandits} is convex.

Finally, we examine the minimization objective \eqref{eq: optimization object bandits}, where $\pi$ and $r$ are fixed and independent of the behavior policy $\mu$. For all $s$, we express the objective function as
\begin{align}
\E_{a\sim\mu}[\rho(a|s)^2r(s,a)^2]=\sum_{a\in\{a|\mu(a|s)>0\}}\frac{\pi(a|s)^2r(s,a)^2}{\mu(a|s)}.
\end{align}
Then, for each $a$,
we decompose the objective function as 
\begin{align}
f_a(\mu(a|s))\doteq\frac{\pi(a|s)^2r(s,a)^2}{\mu(a|s)}.\label{eq: decomposed objective bandits}
\end{align}
Taking the first and second derivatives of $f_a$, we get
\begin{align}
f'_a(\mu(a|s))=-\frac{\pi(a|s)^2r(s,a)^2}{\mu(a|s)^2},\\
f''_a(\mu(a|s))=\frac{2\pi(a|s)^2r(s,a)^2}{\mu(a|s)^3}.
\end{align}
Since $\forall s,a$, $f''_a(\mu(a|s))\geq 0$, we know that \eqref{eq: decomposed objective bandits} is convex for all $a$. Then, as a summation of convex functions, \eqref{eq: optimization object bandits} is also convex.
In conclusion, by the convexity of the feasible set $\mathcal{F}$ and the objective function \eqref{eq: optimization object bandits}, we obtain the convexity of the constrained optimization problem in Lemma~\ref{lemma: convex bandits}.

For feasibility, note that by Lemma~\ref{lemma: optimal bandits}, $\pi\in\mathcal{F}$, which is the feasible set. Thus, we confirm the feasibility in Lemma~\ref{lemma: convex bandits}.

\end{proof}
\subsection{Proof of Lemma~\ref{lem rl pdis unbaised}}
\label{appendix: rl unbias}
\begin{proof}
  We proceed via induction.
  For $t = T-1$,
  we have
  \begin{align}
\E\left[\pdisg(\tau^{\mu_{t:T-1}}_{t:T-1}) \mid S_t\right] =& \E\left[\rho_t R_{t+1} \mid S_t\right] = \E\left[\rho_t q_{\pi, t}(S_t, A_t) \mid S_t \right] \\
    =& \E_{A_t \sim \pi_t(\cdot \mid S_t)}\left[q_{\pi, t}(S_t, A_t) \mid S_t\right] \explain{Lemma~\ref{lemma: unbias bandits}} \\
    =& v_{\pi, t}(S_t).
  \end{align}
  For $t \in [T-2]$,
  we have
\begin{align}
&\E\left[\pdisg(\tau^{\mu_{t:T-1}}_{t:T-1}) \mid S_t\right] \\
=& \E\left[\rho_t R_{t+1} + \rho_t\pdisg(\tau^{\mu_{t+1:T-1}}_{t+1:T-1}) \mid S_t\right] \\
=& \E\left[\rho_t R_{t+1} \mid S_t\right] + \E\left[\rho_t\pdisg(\tau^{\mu_{t+1:T-1}}_{t+1:T-1}) \mid S_t\right] \\
\explain{Law of total expectation}
=& \E\left[\rho_t R_{t+1} \mid S_t\right] + \E_{A_t \sim \mu_t(\cdot \mid S_t), S_{t+1} \sim p(\cdot \mid S_t, A_t)}\left[ \E\left[\rho_t\pdisg(\tau^{\mu_{t+1:T-1}}_{t+1:T-1}) \mid S_t, A_t, S_{t+1}\right] \mid S_t \right] \\
\explain{Conditional independence and Markov property}
=& \E\left[\rho_t R_{t+1} \mid S_t\right] + \E_{A_t \sim \mu_t(\cdot \mid S_t), S_{t+1} \sim p(\cdot \mid S_t, A_t)}\left[ \rho_t \E\left[\pdisg(\tau^{\mu_{t+1:T-1}}_{t+1:T-1}) \mid S_{t+1}\right] \mid S_t \right] \\
=& \E\left[\rho_t R_{t+1} \mid S_t\right] + \E_{A_t \sim \mu_t(\cdot \mid S_t), S_{t+1} \sim p(\cdot \mid S_t, A_t)}\left[ \rho_t v_{\pi, t+1}(S_{t+1}) \mid S_t \right]
\explain{Inductive hypothesis} \\
=& \E_{A_t \sim \mu_t(\cdot \mid S_t)}\left[\rho_t q_{\pi, t}(S_t, A_t) \mid S_t\right] \explain{Definition of $q_{\pi, t}$} \\
=& \E_{A_t \sim \pi_t(\cdot \mid S_t)}\left[q_{\pi, t}(S_t, A_t) \mid S_t\right] \explain{Lemma~\ref{lemma: unbias bandits}} \\
=& v_{\pi, t}(S_t),
\end{align}
  which completes the proof.
\end{proof}

\subsection{Proof of Theorem~\ref{theorem: convex rl}}\label{appendix: convex rl}
\begin{proof}
We first define the set of feasible policies as 
\begin{align}
\textstyle\mathcal{F}\doteq\{\mu\in \Lambda\mid \forall \epsilon,t,s,\E_{a\sim \mu_t}[v^c_{\mu,t}(s )]\leq \delta_{\epsilon,t}(s)\}.\label{eq: feasible set rl}
\end{align}
We begin by examining each constraint.
% For simplicity, for all $s$, we use $\mu_s$ to denote the probability distribution over $\mathcal{A}$ given $s$.
In the first constraint of $\Lambda$ \eqref{eq: definition lambda rl}, 
\begin{align}
 \forall t, s, a, \mu_t(a|s) = 0 \implies
  \pi_t(a|s)q_{\pi, t}(s, a) = 0 .   \label{eq: rl coverage constraint}
\end{align}
The feasible set of \eqref{eq: rl coverage constraint} is a linear subspace of $\mathbb{R}^{|\mathcal{A}|}$ defined by a set of linear equations. Thus, this feasible set is convex.

Next, we decompose the other constraint of $\Lambda$ \eqref{eq: definition lambda rl}, $\mu_t(\cdot|s)\in \Delta^{|\fA| - 1}$, into two constraints:
\begin{align}
\sum_a\mu_t(a|s)=1,\label{eq: sum 1 rl appendix}\\
\forall a, \mu_t(a|s)\geq 0. \label{eq: nonnegative rl appendix}
\end{align}
For all $t$ and $s$, in \eqref{eq: sum 1 rl appendix}, the feasible set can be written as
\begin{align}
\mathbf{1}^\top\overrightarrow{\mu_{s,t}}=1,    \label{eq: sum1 appendix linear rl}
\end{align}
where $\mathbf{1}\in \mathbb{R}^{|\mathcal{A|}}$ is the vector of ones
and $\overrightarrow{\mu_{s,t}}\in \mathbb{R}^{|\mathcal{A}|}$ is defined as
\begin{align}
\overrightarrow{\mu_{s,t}}\doteq
\begin{bmatrix}
\mu_t(a_1|s) \\
\vdots\\
\mu_t(a_{|\mathcal{A}|}|s)
\end{bmatrix}.
\end{align}
Since \eqref{eq: sum1 appendix linear rl} is linear, the feasible set of constraint \eqref{eq: sum 1 rl appendix} is affine and thus convex \citep{boyd2004convexopt}.

For all $t$ and $s$, the feasible set for the constraint in \eqref{eq: nonnegative rl appendix} is the non-negative orthant, defined as
\begin{align}
    \mathbb{R}_+^{|\mathcal{A}|}\doteq \{\mu_t(\cdot|s)\in\mathbb{R}^{|\mathcal{A}|}\mid \mu_t(a|s)\geq 0, \forall a\}.
\end{align}
Since the non-negative orthant forms a convex cone and is known to be a convex set \citep{boyd2004convexopt}, we conclude that this constraint is convex.

Next, we define the vector of the state-action value function for the cost $c$ for each $s$ as
\begin{align}
\mathbf{q}_{\mu,t}\doteq
\begin{bmatrix}
q^c_{\mu,t}(s,a_1) \\
\vdots\\
q^c_{\mu,t}(s,a_{|\mathcal{A}|})
\end{bmatrix}.
\end{align}
Then, for all $\epsilon$, $t$ and $s$, the safety constraint \eqref{eq: safety constraint rl appendix} can be rewritten as
\begin{align}
\mathbf{q}_{\mu,t}^\top \overrightarrow{\mu_{s,t}}\leq \delta_{\epsilon,t}(s),
\end{align}
which is a linear inequality in $\mu_t$. Thus, its feasible set is a convex half-space. Because all the constraints' feasible sets are convex, we conclude that the feasible set $\mathcal{F}$ in \eqref{eq: feasible set rl} is convex.

To prove Theorem~\ref{theorem: convex rl}, we express the objective function as
\begin{align}
\E_{a\sim\mu_t}[\rho_t^2\tilde{r}_t(s,a)]=\sum_{a\in\{a|\mu_t(a|s)>0\}}\frac{\pi_t(a|s)^2\tilde{r}_t(s,a)}{\mu_t(a|s)},\label{eq: decomposed objective rl}
\end{align}
where $\tilde{r}$ in \eqref{eq: extended reward} is defined as 
\begin{align}
\textstyle
\tilde{r}_t(s,a) 
\doteq&
\begin{cases}
r_{\pi,t}(s,a)^2 & t = T-1,\\
\nu_{\pi,t}(s,a)
+ q_{\pi, t}(s,a)^2+\E_{S_{t+1}}\left[\V\left(\pdisg(\tau^{\mu^*_{t+1:T-1}}_{t+1:T-1})\mid S_{t+1}\right) \mid s,a\right]  & t \in [T-2].
\end{cases}
\end{align}
Here, $\tilde{r}_t$ can be learned with logged offline data, as detailed in Algorithm~\ref{alg: safe algorithm}, and it is unrelated to $\mu_t$.
Then, for each $a$,
we decompose the objective function as 
\begin{align}
f_a(\mu_t(a|s))\doteq\frac{\pi_t(a|s)^2\tilde{r}_t(s,a)}{\mu_t(a|s)}.\label{eq: decomposed objective a rl}
\end{align}
Taking the first and second derivatives of $f_a$, we get
\begin{align}
f'_a(\mu_t(a|s))=-\frac{\pi_t(a|s)^2\tilde{r}(s,a)}{\mu_t(a|s)^2},\\
f''_a(\mu_t(a|s))=\frac{2\pi_t(a|s)^2\tilde{r}(s,a)}{\mu_t(a|s)^3}.
\end{align}
Notice that the extended reward $\tilde{r}$ defined in \eqref{eq: extended reward} is non-negative, since all the summands are non-negative.
Thus, $\forall t,s,a$, $f''_a(\mu_t(a|s))\geq 0$, and we know that \eqref{eq: decomposed objective a rl} is convex for all $a$. Then, as a summation of convex functions, \eqref{eq: optimization object rl} is also convex.
In conclusion, by the convexity of the feasible set $\mathcal{F}$ and the objective function \eqref{eq: optimization object rl}, we obtain the convexity of the constrained optimization problem in Theorem~\ref{theorem: convex rl}.

For feasibility, we show that the set of feasible policies \eqref{eq: feasible set rl} is non-empty.
Because $\epsilon \in [0, \infty)$, for the safety constraint, we have
\begin{align}
\E_{a\sim \pi_t}[v^c_{\mu,t}(s)]\leq (1+\epsilon)\E_{a\sim \pi_t}[v^c_{\mu,t}(s)]=\delta_{\epsilon, t}(s).
\end{align}
By the definition of $\Lambda$ \eqref{eq: definition lambda rl}, $\forall t$, $\pi_t\in \Lambda$.  
Therefore, the set of feasible policies \eqref{eq: feasible set rl} is non-empty. Thus, the constrained optimization problem in Theorem~\ref{theorem: convex rl} is feasible.
\end{proof}

\subsection{Proof of Theorem~\ref{theorem: optimal rl}}
\label{appendix: optimal rl}
To prove Theorem~\ref{theorem: optimal rl}, we first restate a recursive expression of the variance $  \V\left(\pdisg(\tau^{\mu_{t:T-1}}_{t:T-1})\mid S_t\right)$ for all $\mu\in \Lambda$ from \citet{liu2024efficient}, and present its proof for completeness.
\begin{lemma}\label{lem:recursive-var}
For any $\mu \in \Lambda$, for $t = T-1$,
\begin{align}
\V\left(\pdisg(\tau^{\mu_{t:T-1}}_{t:T-1})\mid S_t\right) = \E_{A_t \sim \mu_t}\left[\rho_t^2 q_{\pi, t}^2(S_t, A_t) \mid S_t\right] -  v_{\pi, t}^2(S_t),
\end{align}
for $t \in [T-2]$,
\begin{align}
&\V\left(\pdisg(\tau^{\mu_{t:T-1}}_{t:T-1})\mid S_t\right) \\
=& \E_{A_t\sim \mu_t}\left[\rho_t^2 \left(\E_{S_{t+1}}\left[\V\left(\pdisg(\tau^{\mu_{t+1:T-1}}_{t+1:T-1})\mid S_t\right) \mid S_t, A_t\right] + \nu_{\pi,t}(S_t, A_t) + q_{\pi, t}^2(S_t, A_t)\right) \mid S_t\right] \\
&- v_{\pi, t}^2(S_t).
\end{align}
\end{lemma}
\begin{proof}
For completeness, we provide the proof from \citet{liu2024efficient}.
We proceed via induction. 
  When $t =  T-1$, we have
  \begin{align}
  \V\left(\pdisg(\tau^{\mu_{t:T-1}}_{t:T-1})\mid S_t\right) =& \V\left(\rho_t r(S_t, A_t)\mid S_t\right) \\
  =& \V\left(\rho_t q_{\pi, t}(S_t, A_t)\mid S_t\right) \\
  =&  \E_{A_t}\left[\rho_t^2 q_{\pi, t}(S_t, A_t)^2 \mid S_t\right] -  v_{\pi, t}(S_t)^2,
  \end{align}
  
  When $t\in [T-2]$, we have
\begin{align}
\label{eq tmp4}
&\V\left(\pdisg(\tau^{\mu_{t:T-1}}_{t:T-1})\mid S_t\right) \\
=& \E_{A_t}\left[ \V\left(\pdisg(\tau^{\mu_{t:T-1}}_{t:T-1})\mid S_t, A_t\right)\mid S_t\right] + \V_{A_t}\left(\E\left[\pdisg(\tau^{\mu_{t:T-1}}_{t:T-1}) \mid S_t, A_t\right]\mid S_t\right) 
\explain{Law of total variance} \\
% \explain{Law of total variance \eqref{eq:total-variance}} \\
=& \E_{A_t}\left[ \rho_t^2 \V\left(r(S_t,A_t) + \pdisg(\tau^{\mu_{t+1:T-1}}_{t+1:T-1}) \mid S_t, A_t\right)\mid S_t\right] \\
&+ \V_{A_t}\left(\rho_t \E\left[r(S_t,A_t) + \pdisg(\tau^{\mu_{t+1:T-1}}_{t+1:T-1}) \mid S_t, A_t\right]\mid S_t\right)  
\explain{By \eqref{eq:PDIS-recursive}}
\\
=& \E_{A_t}\left[ \rho_t^2 \V\left(\pdisg(\tau^{\mu_{t+1:T-1}}_{t+1:T-1}) \mid S_t, A_t\right)\mid S_t\right] + \V_{A_t}\left(\rho_t \E\left[r(S_t,A_t) + \pdisg(\tau^{\mu_{t+1:T-1}}_{t+1:T-1}) \mid S_t, A_t\right]\mid S_t\right) \explain{Deterministic reward $r$} \\
=& \E_{A_t}\left[ \rho_t^2 \V\left(\pdisg(\tau^{\mu_{t+1:T-1}}_{t+1:T-1}) \mid S_t, A_t\right)\mid S_t\right] + \V_{A_t}\left(\rho_t q_{\pi, t}(S_t, A_t)\mid S_t\right).
\end{align}

Further decomposing the first term, we have
  \begin{align}
    \label{eq tmp3}
  &\V\left(\pdisg(\tau^{\mu_{t+1:T-1}}_{t+1:T-1}) \mid S_t, A_t\right) \\
  =& \E_{S_{t+1}}\left[\V\left(\pdisg(\tau^{\mu_{t+1:T-1}}_{t+1:T-1}) \mid S_t, A_t, S_{t+1}\right) \mid S_t, A_t\right] \\
  &+ \V_{S_{t+1}}\left(\E\left[\pdisg(\tau^{\mu_{t+1:T-1}}_{t+1:T-1}) \mid S_t, A_t, S_{t+1}\right]\mid S_t, A_t\right) 
  \explain{Law of total variance}
  % \explain{Law of total variance \eqref{eq:total-variance}}
  \\
  =& \E_{S_{t+1}}\left[\V\left(\pdisg(\tau^{\mu_{t+1:T-1}}_{t+1:T-1}) \mid S_{t+1}\right) \mid S_t, A_t\right] + \V_{S_{t+1}}\left(\E\left[\pdisg(\tau^{\mu_{t+1:T-1}}_{t+1:T-1}) \mid S_{t+1}\right]\mid S_t, A_t\right) \explain{Markov property} \\
  =& \E_{S_{t+1}}\left[\V\left(\pdisg(\tau^{\mu_{t+1:T-1}}_{t+1:T-1}) \mid S_{t+1}\right) \mid S_t, A_t\right] + \V_{S_{t+1}}\left(v_{\pi, t+1}(S_{t+1})\mid S_t, A_t\right). \explain{Lemma~\ref{lem rl pdis unbaised}}
  \end{align}
 Then, plugging~\eqref{eq tmp3} back to~\eqref{eq tmp4} yields
  \begin{align}
    &\V\left(\pdisg(\tau^{\mu_{t:T-1}}_{t:T-1})\mid S_t\right) \\
    =&\E_{A_t}\left[\rho_t^2 \left(\E_{S_{t+1}}\left[\V\left(\pdisg(\tau^{\mu_{t+1:T-1}}_{t+1:T-1}) \mid S_{t+1}\right) \mid S_t, A_t\right] + \V_{S_{t+1}}(v_{\pi,t}(S_{t+1})\mid S_t=s,A_t=a)\right) \mid S_t\right] \\
    &+ \V_{A_t}\left(\rho_t q_{\pi, t}(S_t, A_t)\mid S_t\right) \\
    =&\E_{A_t}\left[\rho_t^2 \left(\E_{S_{t+1}}\left[\V\left(\pdisg(\tau^{\mu_{t+1:T-1}}_{t+1:T-1}) \mid S_{t+1}\right) \mid S_t, A_t\right] + \V_{S_{t+1}}(v_{\pi,t}(S_{t+1})\mid S_t=s,A_t=a)\right) \mid S_t\right] \\
    &+ \E_{A_t}\left[\rho_t^2 q_{\pi, t}(S_t, A_t)^2\mid S_t\right] - \left(\E_{A_t}\left[\rho_t q_{\pi, t}(S_t, A_t) \mid S_t\right]\right)^2 \\
    =&\E_{A_t}\left[\rho_t^2 \left(\E_{S_{t+1}}\left[\V\left(\pdisg(\tau^{\mu_{t+1:T-1}}_{t+1:T-1}) \mid S_{t+1}\right) \mid S_t, A_t\right] + \V_{S_{t+1}}(v_{\pi,t}(S_{t+1})\mid S_t=s,A_t=a)\right) \mid S_t\right] \\
    &+ \E_{A_t}\left[\rho_t^2 q_{\pi, t}(S_t, A_t)^2\mid S_t\right] -  v_{\pi, t}(S_t)^2,\explain{Lemma~\ref{lemma: unbias bandits}}\\
=&\E_{A_t}\left[\rho_t^2 \left(\E_{S_{t+1}}\left[\V\left(\pdisg(\tau^{\mu_{t+1:T-1}}_{t+1:T-1}) \mid S_{t+1}\right) \mid S_t, A_t\right] + \nu_{\pi,t}(S_t,A_t)+q_{\pi, t}(S_t, A_t)^2\right) \mid S_t\right] \\
& -  v_{\pi, t}(S_t)^2,\explain{Definition of $\nu$}\\
  \end{align}
  which completes the proof.
  \end{proof}
  Then, with the extended reward $\tilde{r}$ in \eqref{eq: extended reward} defined as
\begin{align}
\textstyle
\tilde{r}_t(s,a) 
\doteq&
\begin{cases}
r_{\pi,t}(s,a)^2 & t = T-1,\\
\nu_{\pi,t}(s,a)
+ q_{\pi, t}(s,a)^2+\E_{S_{t+1}}\left[\V\left(\pdisg(\tau^{\mu^*_{t+1:T-1}}_{t+1:T-1})\mid S_{t+1}\right) \mid s,a\right]  & t \in [T-2],
\end{cases}
\end{align}
we can express the variance in a succinct form
\begin{align}
\V\left(\pdisg(\tau^{\mu^*_{t:T-1}}_{t:T-1})\mid S_t=s\right) =  \E_{a\sim\mu}[\rho_t^2\tilde{r}_t(s,a)]- v_{\pi,t}(s)^2, \quad \forall s,t. \label{eq: variance r - v}
\end{align}
Now, we restate Theorem~\ref{theorem: optimal rl} and present its proof.
\restatestepwiseoptimal*
In Appendix~\ref{appendix: convex rl}, we show that $\forall t$, $\pi_t\in \mathcal{F}$, where $\mathcal{F}$ in \eqref{eq: feasible set rl} is the set of feasible policies for the constrained optimization problem in Theorem~\ref{theorem: convex rl}. Recall that $\mu_t^*$ is defined as the optimal solution to the problem \eqref{eq: optimization object rl}, i.e.,
\begin{align}
\mu^*_t \doteq&\argmin_{\mu_t\in \mathcal{F}} \textstyle \E_{a\sim\mu_t}[\rho^2_t\tilde{r}(s,a)]. \label{eq: mu star argmin rl}
\end{align}
Thus, $\forall t,s$,
\begin{align}
&\V\left(\pdisg(\tau^{\mu^*_{t:T-1}}_{t:T-1})\mid S_t=s\right) \\
=&\E_{a\sim\mu^*_t}[\rho_t^2\tilde{r}_t(s,a)]- v_{\pi,t}(s)^2\explain{By \eqref{eq: variance r - v} }\\
\leq &\E_{a\sim\pi_t}[\rho_t^2\tilde{r}_t(s,a)]- v_{\pi,t}(s)^2\explain{By \eqref{eq: mu star argmin rl} and $\pi_t\in \mathcal{F}$}\\
=&\V\left(\pdisg(\tau^{\pi_{t:T-1}}_{t:T-1})\mid S_t=s\right), \explain{By \eqref{eq: variance r - v} }
\end{align}
which completes the proof.

\subsection{Proof of Theorem~\ref{theorem: optimal rl final}}
\label{appendix: optimal rl final}
\begin{proof}
We first prove the variance reduction property.
\begin{align}
 &\V\left(\pdisg(\tau^{\mu^*_{0:T-1}}_{0:T-1})\right) \\
=& \E_{S_0}\left[\V\left(\pdisg(\tau^{\mu^*_{0:T-1}}_{0:T-1}) \mid S_0\right)\right] + \V_{S_0}\left(\E\left[\pdisg(\tau^{\mu^*_{0:T-1}}_{0:T-1}) \mid S_0\right]\right) \explain{Law of Total Variance}\\
=& \E_{S_0}\left[\V\left(\pdisg(\tau^{\mu^*_{0:T-1}}_{0:T-1}) \mid S_0\right)\right] + \V_{S_0}\left(v_{\pi, 0}(S_0)\right)  \explain{By Lemma~\ref{lem rl pdis unbaised} and $\mu^*\in \Lambda$}\\
\leq &\E_{S_0}\left[\V\left(\pdisg(\tau^{\pi_{0:T-1}}_{0:T-1}) \mid S_0\right)\right] + \V_{S_0}\left(v_{\pi, 0}(S_0)\right)\explain{Theorem~\ref{theorem: optimal rl}}\\
=&\E_{S_0}\left[\V\left(\pdisg(\tau^{\pi_{0:T-1}}_{0:T-1}) \mid S_0\right)\right] + \V_{S_0}\left(\E\left[\pdisg(\tau^{\pi_{0:T-1}}_{0:T-1}) \mid S_0\right]\right) \explain{By Lemma~\ref{lem rl pdis unbaised} and $\pi\in \Lambda$}\\
=&\V\left(\pdisg(\tau^{\pi_{0:T-1}}_{0:T-1})\right). \explain{Law of Total Variance}
\end{align}
Next, we prove the safety constraint satisfaction.
\begin{align}
&J^c(\mu^*)\\
=&\sum_sp_0(s)v^c_{\mu^*,0}(s)\\
=&\sum_sp_0(s)
\E_{a\sim\mu^*_0}[q^c_{\mu^*,0}(s,a)]\\
\leq &\sum_sp_0(s)
\delta_{\epsilon,0}(s)\explain{Theorem~\ref{theorem: convex rl}}\\
=&\sum_sp_0(s)(1+\epsilon)v^c_{\pi,0}(s)\explain{By \eqref{eq: constraint j}}\\
=&(1+\epsilon)\sum_sp_0(s)v^c_{\pi,0}(s)\\
=&(1+\epsilon)J^c(\pi),
\end{align}
which completes the proof.
\end{proof}
\subsection{Proof of Lemma~\ref{lemma: recursive tilde r}}\label{appendix: recursive tilde r}
\begin{proof}
$\forall s,a$, when $t=T-1$, $\tilde{r}_t(s,a)=r_{\pi, t}(s,a)^2 $, as defined in \eqref{eq: extended reward}. For $t\in[T-2]$,
\begin{align}
&\tilde{r}_t(s,a)\\    
=&\nu_{\pi,t}(s,a)
+ q_{\pi, t}(s,a)^2+\E_{S_{t+1}}\left[\V\left(\pdisg(\tau^{\mu^*_{t+1:T-1}}_{t+1:T-1})\mid S_{t+1}\right) \mid s,a\right] \explain{By \eqref{eq: extended reward}}\\
=&\nu_{\pi,t}(s,a)
+ q_{\pi, t}(s,a)^2\\
&+\sum_{s'}p(s'|s,a)\left[\E_{A_{t+1}\sim \mu^*_{t+1}}\left[\rho_{t+1}^2 \left(\E_{S_{t+2}}\left[\V\left(\pdisg(\tau^{\mu^*_{t+2:T-1}}_{t+2:T-1})\mid S_{t+2}\right) \mid S_{t+1}, A_{t+1}\right] \right.\right.\right.\\
&\left.\left.\left.+ \nu_{\pi,t+1}(S_{t+1},A_{t+1})+ q_{\pi, t+1}(S_{t+1}, A_{t+1})^2\right) \mid S_{t+1}=s'\right]- v_{\pi, t+1}(s')^2\right] \explain{By Lemma~\ref{lem:recursive-var}}\\
=&\nu_{\pi,t}(s,a)
+ q_{\pi, t}(s,a)^2+\sum_{s'}p(s'|s,a)\left[\E_{A_{t+1}\sim\mu^*_{t+1}}\qty[\rho_{t+1}^2\tilde{r}_{\pi,t+1}(S_{t+1},A_{t+1})\mid S_{t+1}=s']\right.\\
&\left.- v_{\pi, t+1}(s')^2 \right]\explain{By \eqref{eq: extended reward}}\\
=&\nu_{\pi,t}(s,a)
+ q_{\pi, t}(s,a)^2+\sum_{s',a'}p(s'|s,a)\qty[\rho_{t+1}\pi_{t+1}(a'|s')\tilde{r}_{\pi,t+1}(s',a')-v_{\pi, t+1}(s')^2 ].\\
=& \V_{S_{t+1}}\left(v_{\pi, t+1}(S_{t+1})\mid S_t=s, A_t=a\right)+ q_{\pi, t}(s,a)^2\\
&+\sum_{s',a'}p(s'|s,a)\qty[\rho_{t+1}\pi_{t+1}(a'|s')\tilde{r}_{\pi,t+1}(s',a')-v_{\pi, t+1}(s')^2 ]\explain{Definition of $\nu$}\\
=& \E_{S_{t+1}}\left[v_{\pi, t+1}(S_{t+1})^2 \mid S_t=s, A_t=a\right] - \E_{S_{t+1}}\left[v_{\pi, t+1}(S_{t+1}) \mid S_t=s, A_t=a\right]^2 \\
&+ q_{\pi, t}(s,a)^2+\sum_{s',a'}p(s'|s,a)\qty[\rho_{t+1}\pi_{t+1}(a'|s')\tilde{r}_{\pi,t+1}(s',a')-v_{\pi, t+1}(s')^2 ]\\
=&\sum_{s'}p(s'|s,a)v_{\pi,t+1}(s')^2 - (q_{\pi,t}(s,a) - r(s,a))^2+ q_{\pi, t}(s,a)^2\\
&+\sum_{s',a'}p(s'|s,a)\rho_{t+1}\pi_{t+1}(a'|s')\tilde{r}_{\pi,t+1}(s',a')-\sum_{s'}p(s'|s,a)v_{\pi, t+1}(s')^2 \\
=&2q_{\pi,t}(s,a)r(s,a)- r(s,a)^2+\sum_{s',a'}p(s'|s,a)\rho_{t+1}\pi_{t+1}(a'|s')\tilde{r}_{\pi,t+1}(s',a')\\
=&2q_{\pi,t}(s,a)r(s,a)- r(s,a)^2+\E_{s' \sim p, a' \sim \pi}\qty[\frac{\pi_{t+1}(a'|s')}{\mu^*_{t+1}(a'|s')}\tilde{r}_{\pi,t+1}(s',a')] .
\end{align}
\end{proof}

\section{Experiment Details}
\label{append: experiment}
\subsection{GridWorld}
\begin{table}[h]
\centering
\begin{tabular}{lllll}
\toprule
Environment Size & On-policy MC & \textbf{Ours} & ODI & ROS \\
\midrule
1,000 & 1.000 & \textbf{0.547} & 0.460 & 0.953 \\
27,000 & 1.000 & \textbf{0.575} & 0.484 & 0.987 \\
\bottomrule
\end{tabular}
\caption{Relative variance for estimators on Gridworld. The relative variance is defined as the variance of each estimator divided by the variance of the on-policy Monte Carlo estimator. Numbers are averaged over 900 independent runs (30 target policies, each having 30 independent runs). Standard errors are plotted in Figure \ref{fig:gridworld}.
}
\label{table: gridworld variance}
\end{table}

\begin{table}[h]
\centering
\begin{tabular}{llllll}
\toprule
 Env Size& On-policy MC & \textbf{Ours} & ODI & ROS & Saved Cost Percentage \\
\midrule
10 & 1000 & \textbf{472} & 738 & 1035 & (1000 - 472)/1000 = \textbf{52.8\%} \\
30 & 1000 & \textbf{487} & 765 & 1049 & (1000 - 487)/1000 = \textbf{51.3\%} \\
\bottomrule
\end{tabular}
\caption{Cost needed to achieve the same estimation accuracy that on-policy Monte Carlo achieves with $1000$ episodes on Gridworld. Each number is averaged over 900 runs. Standard errors are plotted in Figure \ref{fig:gridworld_cost}.
}
\label{table: gridworld trajectory cost}
\end{table}
We conduct experiments on Gridworlds with $n^3=1,000$ and $n^3=27,000$ states, where for a Gridworld with size $n^3$, we set the width, height, and time horizon $T$ all to be $n$. The action space contains four different possible actions: up, down, left, and right. After taking an action, the agent has a probability of $0.9$ to move accordingly and a probability of $0.1$ to move uniformly at random. When the agent runs into a boundary, it stays in its current position. We randomly generate the reward function $r(s,a)$ and cost function $c(s,a)$. 
We consider $30$ randomly generated target policies with various performances.
The ground truth policy performance is estimated by the on-policy Monte Carlo method, running each target policy for $10^6$ episodes.
We experiment two different sizes of the Gridworld with a number of $1,000$ and $27,000$ states.

The offline dataset of each environment contains a total of $1,000$ episodes generated by $30$ policies with various performances. The performance of those policies ranges from completely random initialized policies to well-trained policies in each environment. 
For example, in Hopper, the performance of those $30$ policies ranges from around $18$ to around $2800$. We let offline data be generated by various policies to simulate the fact that offline data are from different past collections.

We learn functions $q_{\pi,t}, q^c_{\pi,t}$, and $\hat{r}_{\pi,t}$ using Fitted Q-Evaluation algorithms (FQE, \citet{le2019batch}) by passing data tuples in $\mathcal{D}_{\nu}$ from $t=T-1$ to $0$.
It is worth noticing that Fitted Q-Evaluation (FQE, \citet{le2019batch}) is a different algorithm from Fitted Q-Improvement (FQI). Importantly, Fitted Q-Evaluation is not prone to overestimate the action-value function $q_{\pi,t}$ because it does not have any $\max$ operator and does not change the policy. 
All hyperparameters are tuned offline
based on Fitted Q-learning loss. We leverage a one-hidden-layer neural network and test the neural network size with $[64,128,256]$. We then choose $64$ as the final size. 
We also test the learning rate for Adam optimizer with $[1\text{e}^{-5},1\text{e}^{-4},1\text{e}^{-3},1\text{e}^{-2}]$ and finally choose to use the default learning rate $1\text{e}^{-3}$ as learning rate for Adam optimizer \citep{kingma2014adam}.
For the benchmark algorithms, we use their reported hyperparameters \citep{zhong2022robust, liu2024efficient}.
Each policy has 30 independent runs, resulting in a total of $30\times 30=900$ runs.
Thus, each curve in Figure \ref{fig:gridworld}, Figure~\ref{fig:gridworld_cost} and each number in Table~\ref{table: gridworld cost}, Table~\ref{table: gridworld variance} and Table~\ref{table: gridworld trajectory cost} are averaged from 900 different runs over a wide range of policies, demonstrating a strong statistical significance.

\subsection{MuJoCo}

\begin{figure}[ht]
\begin{minipage}{0.18\textwidth}
\centering
\includegraphics[width=1\textwidth]{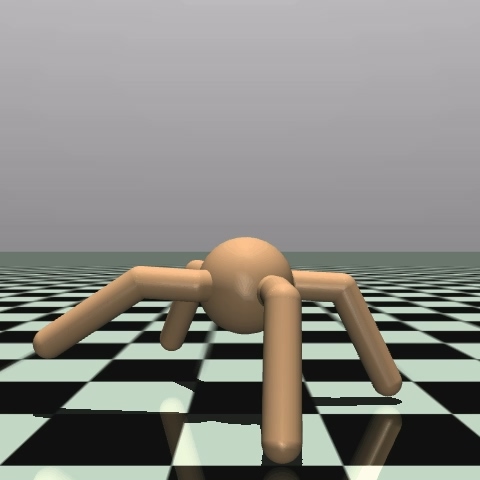}
\end{minipage}
\begin{minipage}{0.18\textwidth}
\centering \includegraphics[width=1\textwidth]{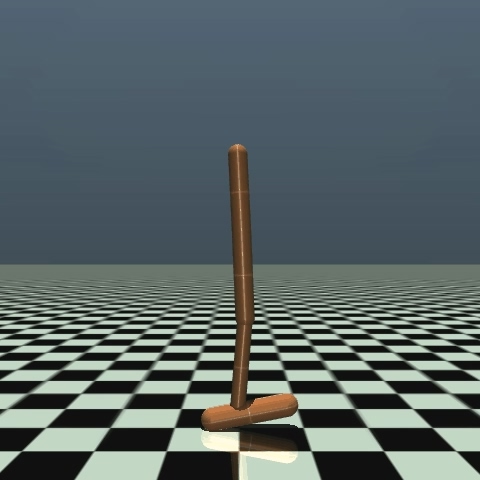}
\end{minipage}
\begin{minipage}{0.18\textwidth}
\centering \includegraphics[width=1\textwidth]{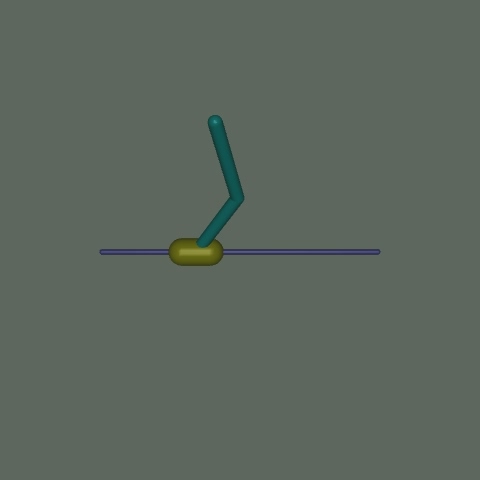}
\end{minipage}
\begin{minipage}{0.18\textwidth}
\centering \includegraphics[width=1\textwidth]{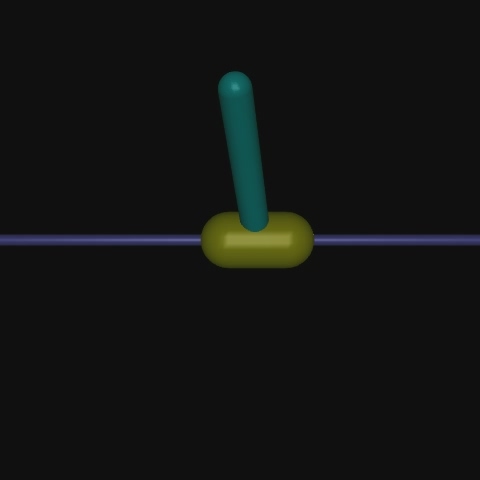}
\end{minipage}
\begin{minipage}{0.18\textwidth}
\centering \includegraphics[width=1\textwidth]{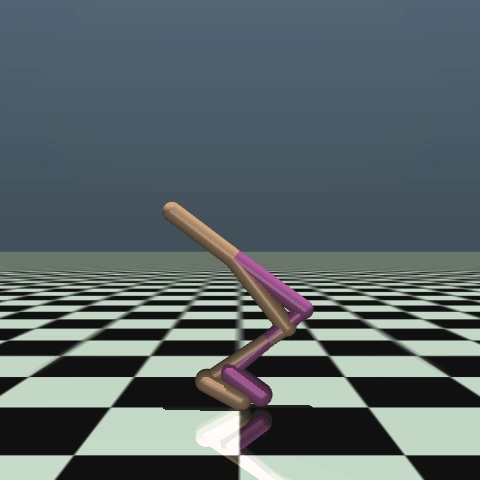}
\end{minipage}
\centering
\caption{
MuJoCo robot simulation tasks \citep{todorov2012mujoco}. Pictures are adapted from \citep{liu2024efficient}.
Environments from the left to the right are Ant, Hopper, InvertedDoublePendulum,  InvertedPendulum, and Walker.
} 
\label{fig:cart_pole_image}
\end{figure}
\begin{figure}[t]
\includegraphics[width=1\textwidth]{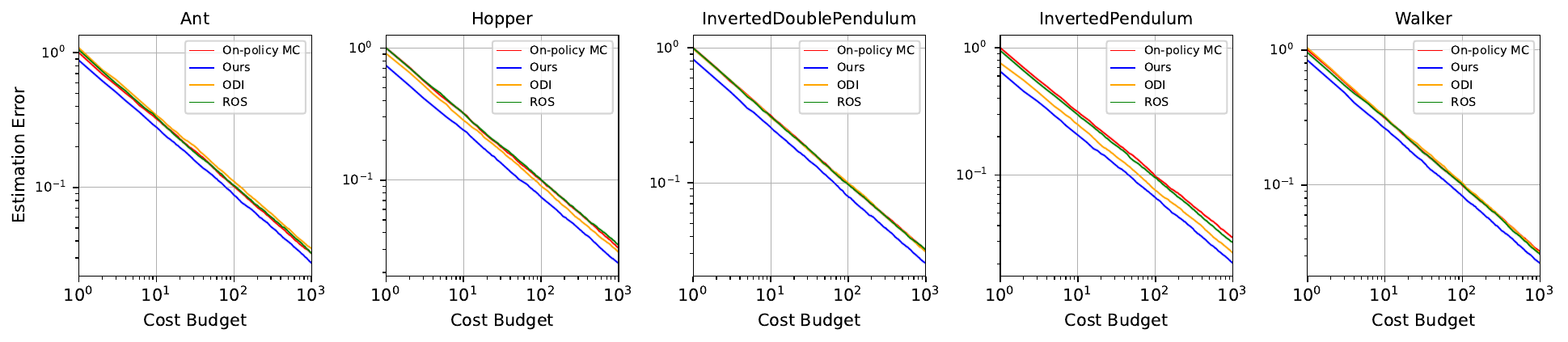}
\centering
\caption{
Results on MuJoCo with log-scale y-axis to show the error does not converge. 
Each curve is averaged over 900 runs (30 target policies, each having 30 independent runs). 
Shaded regions denote standard errors and are invisible for some curves because they are too small.
}
\label{fig:mujoco 1000}
\end{figure}

\begin{table}[H]
\begin{center}
\begin{small}
\begin{tabular}{lllll}
\toprule
& On-policy MC & Ours  & ODI  & ROS   \\
\midrule
Ant & 1.000 & \textbf{0.835} & 0.811 & 1.032 \\
Hopper & 1.000 & \textbf{0.596} & 0.542 & 1.005 \\
I. D. Pendulum & 1.000 & \textbf{0.778} & 0.724 & 0.992 \\
I. Pendulum & 1.000 & \textbf{0.439} & 0.351 & 0.900 \\
Walker & 1.000 & \textbf{0.728} & 0.696 & 0.908 \\
\bottomrule
\end{tabular}
\end{small}
\end{center}
\caption{Relative variance of estimators on MuJoCo. The relative variance is defined as the variance of each estimator divided by the variance of the on-policy Monte Carlo estimator. All numbers are averaged over 900 independent runs (30 target policies, each having 30 independent runs).}
\label{table: mujoco variance}
\end{table}

\begin{table}[H]
\begin{center}
\begin{small}
\begin{tabular}{lllll}
\toprule
& On-policy MC & Ours  & ODI  & ROS   \\
\midrule
Ant & 1.000 & \textbf{0.897} & 1.397 & 1.033 \\
Hopper & 1.000 & \textbf{0.930} & 1.523 & 1.021 \\
I. D. Pendulum & 1.000 & \textbf{0.876} & 1.399 & 1.012 \\
I. Pendulum & 1.000 & \textbf{0.961} & 1.743 & 0.990 \\
Walker & 1.000 & \textbf{0.953} & 1.485 & 1.061 \\
\bottomrule
\end{tabular}
\end{small}
\end{center}
\caption{Average trajectory cost on MuJoCo. Numbers are normalized by the cost of the on-policy estimator. ODI and ROS have much larger costs because they both ignore safety constraints. \textbf{Our method is the only method consistently achieving both variance reduction and safety constraint satisfaction.}}
\label{table: mujoco single cost}
\end{table}

MuJoCo is a physics engine with various stochastic environments, in which the goal is to control a robot to achieve different behaviors such as walking, jumping, and balancing. 
We construct $30$ policies in each environment, resulting a total of $150$ policies. The policies demonstrate a wide range of performance generated by 
the proximal policy optimization (PPO) algorithm \citep{schulman2017proximal} using the default PPO implementation in \citet{huang2022cleanrl}. 
Original MuJoCo environments are Markov decision processes (MDP) and do not have cost functions. We enhance it with cost functions to make it constrained Markov decision processes (CMDP). Specifically, the cost of the MuJoCo environments is built on the control cost of the robot. The control cost is the L2 norm of the action and is proposed by OpenAI Gymnasium \citep{openai2016gym}. This control cost is motivated by the fact that large actions in robots induce sudden changes in the robot's state and may cause safety issues.

We set each environment in MuJuCo to have a fixed time horizon $100$ in OpenAI Gymnasium \citep{towers2024gymnasium}.
Because our methods are designed for discrete action space,
we discretize the first dimension of the MuJoCo action space.
The remaining dimensions are then controlled by the PPO policies and are deemed as part of the environment.
The offline dataset for each environment contains $1,000$ episodes generated by $30$ policies with various performances, following the same method as in the Gridworld environments.
Functions $q_{\pi,t}, q^c_{\pi,t}$, and $\hat{r}_{\pi,t}$  are learned using the same way as in Gridworld environments. 
Notably, our algorithm is robust on hyperparameters, as all hyperparameters in Algorithm \ref{alg: safe algorithm} are tuned offline and are the same across all MuJoCo and Gridworld experiments.
Each policy in MuJoCo has 30 independent runs, resulting in a total of $30 \times 30 = 900$ runs.
As a result, curves in all figures are averaged from 900 different runs with a wide range of policies, showing a strong statistical significance.

\end{document}